%% file: Formatting-Instructions-LaTeX-2025.tex
\documentclass[letterpaper]{article} 
\usepackage{aaai25}  
\usepackage{times}  
\usepackage{helvet}  
\usepackage{courier}  
\usepackage[hyphens]{url}  
\usepackage{graphicx} 
\usepackage{natbib}  
\usepackage{caption} 
\usepackage{algorithm}
\usepackage{algorithmic}

\usepackage{newfloat}
\usepackage{listings}
\DeclareCaptionStyle{ruled}{labelfont=normalfont,labelsep=colon,strut=off} 
\floatstyle{ruled}
\newfloat{listing}{tb}{lst}{}
\floatname{listing}{Listing}
\usepackage{hyperref}
\usepackage{xcolor}

\definecolor{myred}{rgb}{0.8, 0.2, 0.2}

\hypersetup{
  colorlinks   = true, 
  urlcolor     = blue, 
  linkcolor    = {myred!75!black}, 
  citecolor   = {blue!50!black} 
}
\title{From Logistic Regression to the Perceptron Algorithm: \\ Exploring Gradient Descent with Large Step Sizes}

\iftrue
\author {
    Alexander Tyurin
}
\affiliations{
    AIRI, Moscow, Russia \\
    Skoltech, Moscow, Russia \\
    alexandertiurin@gmail.com \\
}
\fi

\usepackage{bibentry}

\usepackage{amsthm}
\usepackage{amsmath,amsfonts,bm}

\newcommand{\eqdef}{:=} 

\newcommand{\norm}[1]{\left\| #1 \right\|}

\newcommand{\inp}[2]{\left\langle#1,#2\right\rangle} 
\newcommand{\abs}[1]{\left| #1 \right|}

\newcommand{\R}{\mathbb{R}} 
\newcommand{\N}{\mathbb{N}} 

\newcommand{\cN}{\mathcal{N}}

\newcommand{\cO}{\mathcal{O}}

\newcommand*{\refalgone}[1]{\ref{#1}}

\usepackage{nicefrac}
\usepackage{subcaption}
\usepackage{thm-restate}

\theoremstyle{plain}
\newtheorem{theorem}{Theorem}[section]

\theoremstyle{definition}

\newtheorem{assumption}[theorem]{Assumption}
\theoremstyle{remark}

\begin{document}

\maketitle

\begin{abstract}
    We focus on the classification problem with a separable dataset, one of the most important and classical problems from machine learning.  The standard approach to this task is \emph{logistic regression with gradient descent} (LR+GD). Recent studies have observed that LR+GD can find a solution with arbitrarily large step sizes, defying conventional optimization theory. Our work investigates this phenomenon and makes three interconnected key observations about LR+GD with large step sizes.
    First, we find a remarkably simple explanation of why LR+GD with large step sizes solves the classification problem: LR+GD reduces to a batch version of the celebrated perceptron algorithm 
    when the step size $\gamma \to \infty.$ 
    Second, we observe that larger step sizes lead LR+GD to \emph{higher} logistic losses when it tends to the perceptron algorithm, but larger step sizes also lead to \emph{faster} convergence to a solution for the classification problem, meaning that logistic loss is an unreliable metric of the proximity to a solution. Surprisingly, high loss values can actually indicate faster convergence. Third, since the convergence rate in terms of loss function values of LR+GD is unreliable, we examine the iteration complexity required by LR+GD with large step sizes to solve the classification problem and prove that this complexity is suboptimal. To address this, we propose a new method, Normalized LR+GD---based on the connection between LR+GD and the perceptron algorithm---with much better theoretical guarantees.
\end{abstract}

\section{Introduction}
We consider the classical classification problem from machine learning\footnote{In this paper, we examine five algorithms, referred to as \refalgone{eq:gd}, \refalgone{eq:batch_perceptron}, \refalgone{eq:perceptron}, \refalgone{eq:norm_perceptron}, and \refalgone{eq:norm_gd}.} with a dataset $\{(a_i, y_i)\}_{i=1}^n$ and two classes, where $a_i \in \R^d$ and $y_i \in \{-1, 1\}$ for all $i \in [n] \eqdef \{1, \dots, n\}.$ The goal of the classification problem is to
\begin{align}
    \label{eq:linear_task}
    \textnormal{find a vector $\theta \in \R^d$ such that } y_i a_i^\top \theta > 0 \quad \forall i \in [n].
\end{align}
This is the supervised learning problem that finds a linear model (hyperplane) that separates the dataset. In general, this problem is infeasible, and one can easily find an example when the dataset is not linearly separable. We focus on the setup where the data is separable, which is formalized by the assumption:
\begin{assumption}
\label{ass:separable}
\begin{align*}
    \mu \eqdef \max_{\norm{\theta} = 1} \min_{i \in [n]} y_i a_i^\top \theta > 0.
\end{align*}
\end{assumption}
This condition ensures that for some $\theta \in \R^d,$ the dataset can be perfectly classified. The quantity $\mu$ is a \emph{margin} \citep{novikoff1962convergence,pattern}, which characterizes the distance between the two classes. This assumption is practical in modern machine learning problems \citep{soudry2018implicit,ji2018risk}. Albeit it is mostly attributed to large-scale nonlinear models \citep{brown2020language}, where the number of parameters is huge, the analysis of the methods in the linear case is equally important as it serves as a foundation for the nonlinear case. Let us define $R \eqdef \max_{i \in [n]} \norm{a_i}.$

There is a huge number of ways \citep{bishop2006pattern} how one can solve the problem, including support vector machines (SVMs) \citep{cortes1995support}, logistic regression, and the perceptron algorithm \citep{novikoff1962convergence}. This work focuses on the latter two, starting with logistic regression, which can be formalized by the following optimization problem:
\begin{align}
    \label{eq:logistic_regression}
    f(\theta) \eqdef \frac{1}{n} \sum_{i=1}^{n} \log\left(1 + \exp(-y_i a_i^\top \theta)\right) \rightarrow \min_{\theta \in \R^d}.
\end{align}
This optimization problem does not have a finite minimum when the data is separable. Indeed, if $\theta$ separates the dataset, then $y_i a_i^\top \theta > 0$ for all $i \in [n]$ and $f(c \cdot \theta) \to 0,$ when $c \to \infty.$

\subsection{Gradient Descent}

The logistic regression problem \eqref{eq:logistic_regression} can be solved with gradient descent (GD) \citep{nesterov2018lectures}, stochastic gradient descent \citep{robbins1951stochastic}, L-BFGS \citep{liu1989limited}, and variance-reduced methods (e.g., SAG, SVRG) \citep{schmidt2017minimizing,johnson2013accelerating}. We consider the GD method, arguably one of the simplest and most well-understood methods:
\begin{align}
    \label{eq:gd}\tag{LR+GD}
    \theta_{t+1} = \theta_{t} - \gamma \nabla f(\theta_t),
\end{align}
where $\theta_0$ is a starting point, $\gamma > 0$ is a step size, and $\nabla f(\theta_t)$ is the gradient of \eqref{eq:logistic_regression} at the point $\theta_t.$ 

\begin{quote}
\textbf{What do we know about GD in the context of logistic regression (\refalgone{eq:gd})? Surprisingly, despite the huge popularity of GD and logistic regression, we still lack a comprehensive understanding.}
\end{quote}

\subsection{Previous Work}
\textbf{Classical convex and nonconvex optimization theory.} Let us recall the classical result for GD: it is well-known that if $\gamma < \nicefrac{2}{L},$ and a function $f$ is $L$--smooth and lower bounded, which is true for \eqref{eq:logistic_regression}, then\footnote{Note that we can not directly apply the results, for instance, from \citep{nesterov2018lectures} because \eqref{eq:logistic_regression} does not have a \emph{finite} minimum. We need a minor modification of the classical analysis \citep{farfar,ji2018risk}.} $f(\theta_T) - \inf_{\theta \in \R^d} f(\theta) = \widetilde{\cO}\left(\nicefrac{1}{\gamma T}\right)$ \citep{ji2018risk} for convex problems, or $\min_{t \in [T]} \norm{\nabla f(\theta_t)}^2 \leq \cO\left(\nicefrac{1}{\gamma T}\right)$ for nonconvex problems. At the same time, if $\gamma > \nicefrac{2}{L},$ then one can find a $L$--smooth function such that GD diverges (see \citep[Sec.2]{cohen2020gradient}). Under $L$--smothness, the value $\nicefrac{2}{L}$ is special because it divides GD into the \emph{convergence} and \emph{divergence} regimes. \\~\\
\textbf{The edge of stability (EoS) and large step sizes.} Nonetheless, in practice, it was many times observed (e.g., \citep{lewkowycz2020large,cohen2020gradient}) that when a step size is large, $\gamma > \nicefrac{2}{L},$ GD not only not diverges, but non-monotonically with oscillation converges on the task \eqref{eq:logistic_regression}. This phenomenon was coined as \emph{the edge of stability} \citep{cohen2020gradient}. This means that there is something special about the practical machine learning problems. 

The mathematical aspects of the large step size regime have attracted significant attention within the research community, which analyzes the phenomenon through the sharpness of loss functions (the largest eigenvalue of Hessians) \citep{kreisler2023gradient}, small dimension problems \citep{zhu2022understanding,chen2022gradient,ahn2024learning}, bifurcation theory \citep{song2023trajectory}, sharpness behavior in networks with normalization \citep{lyu2022understanding}, 2-layer linear diagonal networks \citep{even2023s}, non-separable data \citep{ji2018risk,meng2024gradient}, self-stabilization \citep{damian2022self, ahn2022understanding,ma2022beyond,wang2022analyzing}. The papers by \citet{wu2024implicit,wu2024large} are the closest to our research since they also analyze GD and logistic regression (\refalgone{eq:gd}). They demonstrate that GD can converge with an arbitrary step size $\gamma > 0.$ 

In particular, the results from \citep{wu2024implicit} show that for any fixed $\gamma > 0$, $f(\theta_t)$ is approximately less than or equal to $\nicefrac{\textnormal{poly}(e^{\gamma})}{t}$ (plus additional terms that depend on other parameters). \citet{wu2024large} refined the dependence on $\gamma$ and demonstrated that GD with a large step size initially operates in \emph{the non-stable regime}, where $f(\theta_t) = \widetilde{\cO}(\nicefrac{(1 + \gamma^2)}{\gamma t})$. After approximately $\widetilde{\Theta}(\nicefrac{\max\{n, \gamma\}}{\mu^2})$ iterations, GD transitions to \emph{the stable regime}, where $f(\theta_t) = \widetilde{\cO}(\nicefrac{1}{\gamma t}).$ By tuning and taking the fixed step size $\gamma = \Theta(T),$ they get the accelerated rate $f(\theta_T) = \widetilde{\cO}(\nicefrac{1}{T^2}).$

\section{Contributions}
This paper delves deeper into understanding the dynamics of \emph{the non-stable regime}, where the loss chaotically oscillates due to large step sizes. We explore logistic regression with gradient descent (\refalgone{eq:gd}) and find that \textbf{1)} \refalgone{eq:gd} reduces to a batch version of the perceptron algorithm (\ref{eq:batch_perceptron}), \textbf{2)} the fastest convergence of \refalgone{eq:gd} to a solution of \eqref{eq:linear_task} is achieved when step sizes and \emph{loss values} are large, and \textbf{3)} \refalgone{eq:gd} is a suboptimal method, does not scale with the number of data points, and can be improved. Let us clarify:

\textbf{1)} We begin with a key observation that the iterates of \refalgone{eq:gd}, when divided by the step size $\gamma,$ converge to the iterates of a batch version \eqref{eq:batch_perceptron} of the celebrated perceptron algorithm \citep{block1962perceptron, novikoff1962convergence} when $\gamma \to \infty.$ In other words, \refalgone{eq:gd} reduces to \ref{eq:batch_perceptron}. The proof of this fact is straightforward and occupies less than half a page. This is an advantage of our paper because the typical proofs on this topic are technical and non-intuitive. When combined with the classical convergence results \citep{novikoff1962convergence}, it offers a clear intuition and explanation for why the method solves \eqref{eq:linear_task} with large step sizes in \emph{the non-stable regime}—a detail that, to our knowledge, has been previously overlooked and nonproven in the literature. (see Section~\ref{sec:gd_is_perceptron})

\begin{figure}[h]
\centering
\begin{subfigure}[t]{0.4\textwidth}
    \centering
    \includegraphics[width=\textwidth]{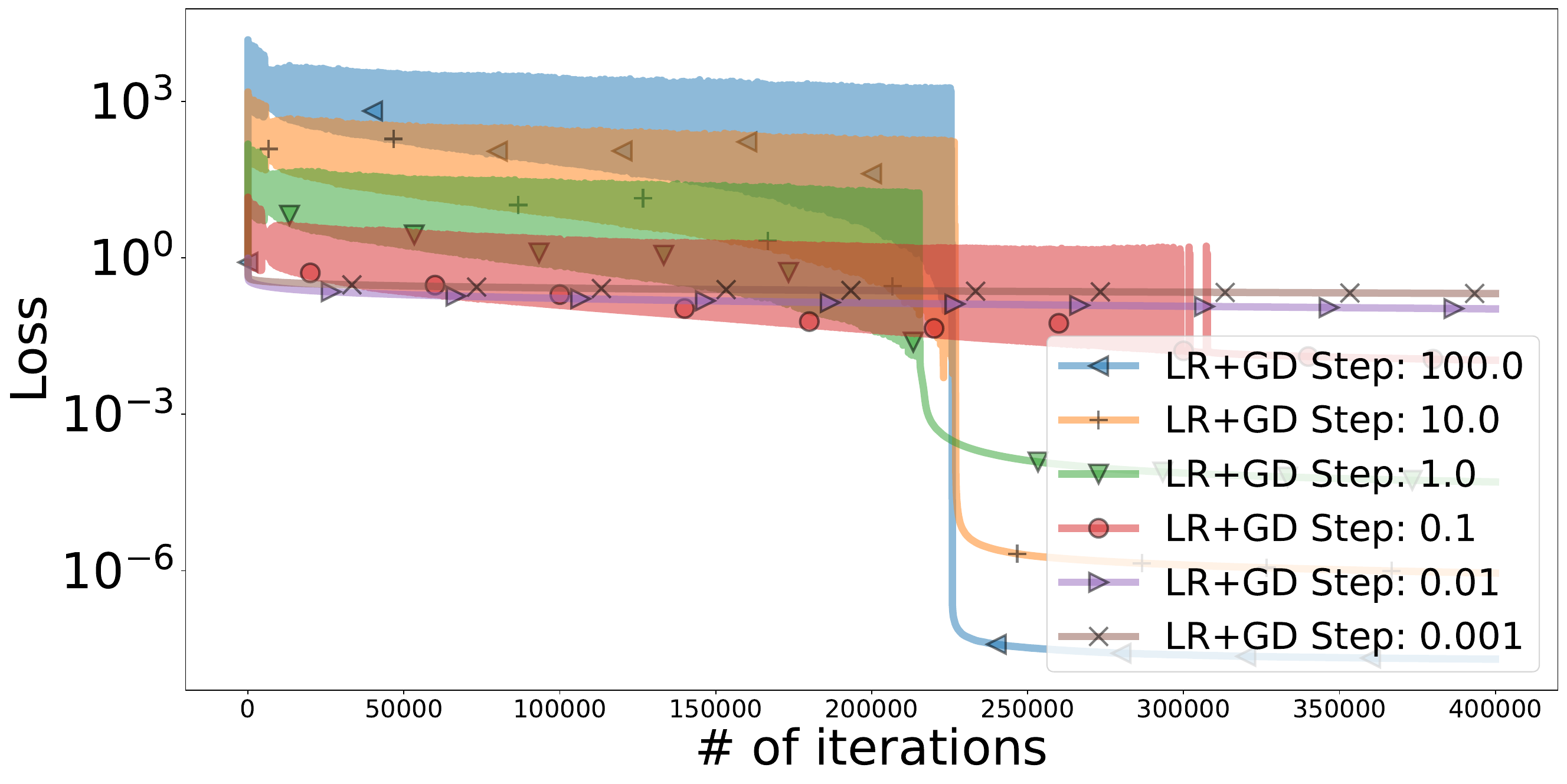}
    \caption{}
    \label{fig:loss}
\end{subfigure}
\begin{subfigure}[t]{0.4\textwidth}
    \centering
    \includegraphics[width=\textwidth]{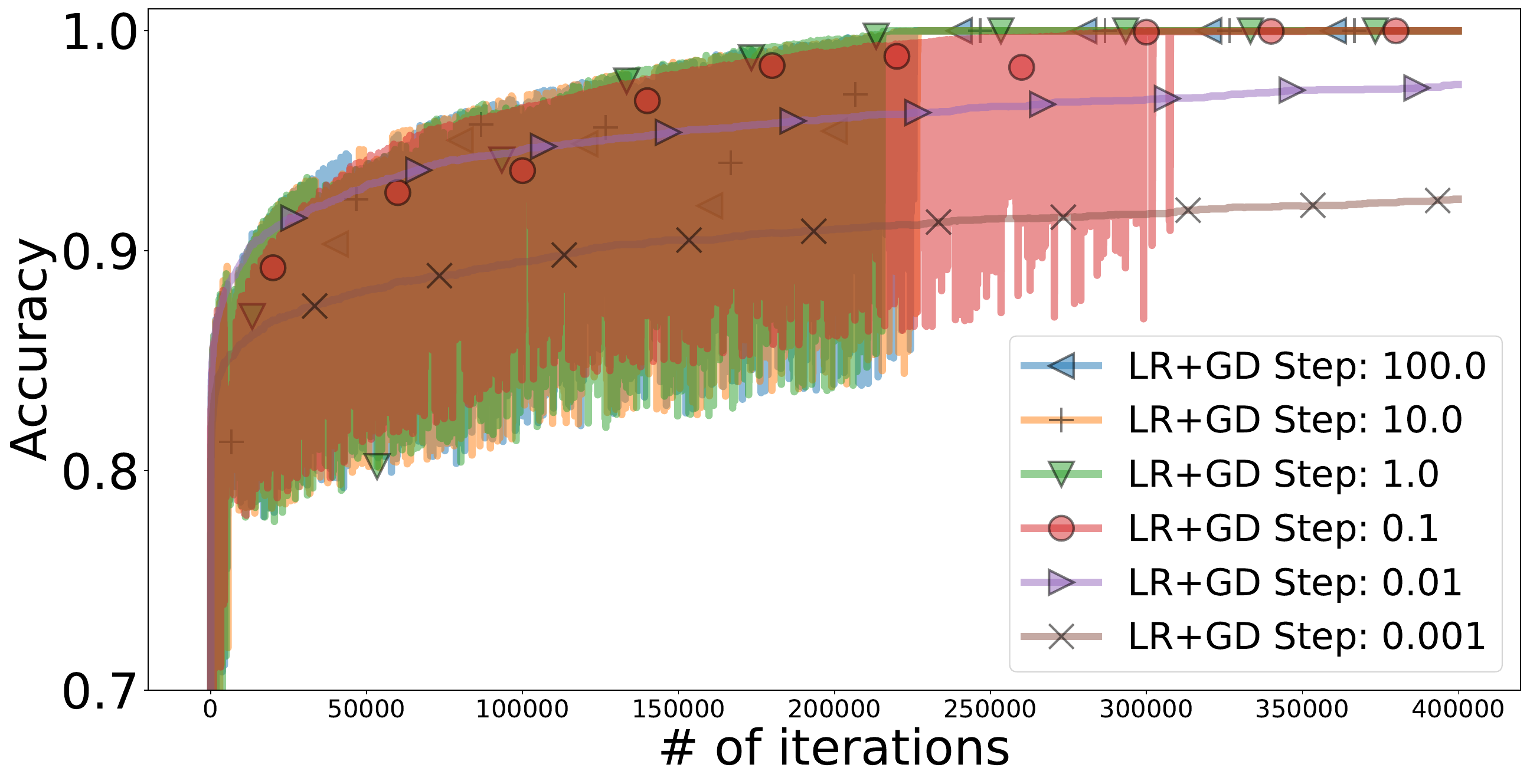}
    \caption{}
    \label{fig:accuracy}
\end{subfigure}
\caption{Illustration on a subset of \emph{CIFAR-10} dataset \citep{krizhevsky2009learning} with $5\,000$ samples and two classes. We run \ref{eq:gd} with various step sizes. Note that there is no randomness involved in the process. Oscillation is a natural behavior of \ref{eq:gd} with separable data and large step sizes.}
\label{fig:main}
\end{figure}

When we test this fact, Theorem~\ref{theorem:reduction}, with numerical experiments (see Figure~\ref{fig:main}), we observe that larger step sizes result in higher loss values (Figure~\ref{fig:loss}) before the moment when \refalgone{eq:gd} attains Accuracy=$1.0.$ Additionally, both loss and accuracy oscillate more (Figure~\ref{fig:loss} and \ref{fig:accuracy}). And despite that, \refalgone{eq:gd} solves \eqref{eq:linear_task} faster with large step sizes. Indeed, notice that \refalgone{eq:gd} has the fastest convergence to Accuracy=$1.0$ with the step sizes $\gamma \in \{1.0, 10.0, 100.0\},$ but at the same time, it has the highest loss values with these steps (more experiments in the next sections and appendix).

\textbf{2)} We investigate this phenomenon further and show that \emph{the logistic loss and the norm of gradients are unreliable metrics}. We argue that the fact that a loss value $f(\theta_t)$ is small does not necessarily indicate that $\theta_t$ is close to solving \eqref{eq:linear_task}. Surprisingly, the opposite can be true. Our experiments and theorem show that high loss values may indicate fast convergence. (see Section~\ref{sec:unrealiable}) 

\textbf{3)} This finding implies that when analyzing and developing methods for solving \eqref{eq:linear_task}, we have to look at the number of iterations required by methods to solve \eqref{eq:linear_task} rather than relying solely on loss and gradient values. Therefore, we looked at the iteration complexity $\nicefrac{n R^2}{\mu^2}$ of \refalgone{eq:gd} with $\gamma \to \infty$ and noticed that it is suboptimal with respect to $n$ since the iteration complexity $\nicefrac{R^2}{\mu^2}$ can be attained by the classical (non-batch) perceptron algorithm (\ref{eq:perceptron}). Moreover, we prove a lower bound, showing that the dependence on $n$ cannot be avoided in \refalgone{eq:gd} with $\gamma \to \infty.$ Provably, \refalgone{eq:gd} is a suboptimal method with large step sizes. Finally, we slightly modify \refalgone{eq:gd} and develop a new method, \refalgone{eq:norm_gd}, basing on the connection between \refalgone{eq:gd} and \refalgone{eq:batch_perceptron}.
This new method provably improves the iteration rate of \refalgone{eq:gd} to $\nicefrac{R^2}{\mu^2}$ when $\gamma \to \infty.$ The new iteration rate to solve \eqref{eq:linear_task} is $n$ times better.

\section{Reduction to the Batch Perceptron Algorithm}
\label{sec:gd_is_perceptron}

Before we state our first result, let us recall a batch version of the perceptron algorithm \citep{novikoff1962convergence,pattern}:
\begin{equation}
    \label{eq:batch_perceptron}\tag{Batch Perceptron}
\begin{aligned}
    &\textnormal{Take the first step } \hat{\theta}_{1} = \hat{\theta}_{0} + \frac{1}{2 n}\sum_{i=1}^n y_i a_i. \\
    &\textnormal{For all $t \geq 1,$ find the set } S_t \eqdef \{i \in [n]\,:y_i \,{a_i}^\top \hat{\theta}_t \leq 0\,\}, \\
    &\textnormal{and take the step } \hat{\theta}_{t+1} = \hat{\theta}_{t} + \frac{1}{n}\sum_{i \in S_t} y_i a_i \textnormal{ while $\abs{S_t} \neq 0,$}
\end{aligned}
\end{equation} 
where $\hat{\theta}_{0}$ is a starting point. This method finds all misclassified samples and uses them to find the next iterate $\hat{\theta}_{t+1}$ of the algorithm. Note that the classical version of the perceptron algorithm does the step only with one misclassified sample, as presented in (\ref{eq:perceptron}). 

We will require the following technical assumption in Theorem~\ref{theorem:reduction}:
\begin{assumption}(Non-Degenerate Dataset)
    \label{ass:bad_dataset}
    For all $j \in [n],$ the hyperplane $\{x \in \R^n:\sum_{i=1}^n x_i \inp{y_i a_i}{y_j a_j} = 0\}$ \emph{does not} intersect the point $(0.5 + k_1, \dots, 0.5 + k_n)$ for all $k_1, ..., k_n \in \N_0.$
\end{assumption}

This is a very weak assumption that cuts off pathological datasets since the chances that any hyperplane will intersect the \emph{countable} set are zero in practice. 
Indeed, assume that $\{x \in \R^n:\sum_{i=1}^n x_i \inp{y_i a_i}{y_j a_j} = 0\}$ intersects some point $(0.5 + k_1, \dots, 0.5 + k_n).$ For any arbitrarily small $\sigma > 0,$ let us take i.i.d.~normal noises $\xi_1, \dots, \xi_n \sim \cN(0, \sigma).$ Then the probability that a slightly perturbed hyperplane $\{x \in \R^n:\sum_{i=1}^n x_i (\inp{y_i a_i}{y_j a_j} + \xi_i) = 0\}$ intersects any point $(0.5 + k_1, \dots, 0.5 + k_n)$ is zero.
We are ready to state and prove the first result:
\begin{theorem}
    \label{theorem:reduction}
    Let Assumption~\ref{ass:separable} hold. For $\gamma \to \infty$ and\footnote{The theorem is true for $\theta_0 \neq 0,$ but we have to modify Assumption~\ref{ass:bad_dataset}, and change $0.5$ to $(1 + \exp(y_i \,a_i^\top \theta_0))^{-1}$.} $\theta_0 = 0,$ the logistic regression with gradient descent (\refalgone{eq:gd}) reduces to the batch perceptron algorithm (\refalgone{eq:batch_perceptron}), i.e., $\theta_t / \gamma \to \hat{\theta}_t$ for all $t \geq 0,$ with 
    $\hat{\theta}_0 = 0$ 
    if the dataset satisfies Assumption~\ref{ass:bad_dataset} (almost all datasets).
\end{theorem}
\begin{proof}
    Clearly, we have 
    \begin{align}
        \label{eq:grad}
        \nabla f(\theta) = - \frac{1}{n} \sum_{i=1}^n (1 + \exp(y_i \,a_i^\top \theta))^{-1} y_i a_i
    \end{align} and 
    $\theta_1 = \theta_0 - \gamma \nabla f(\theta_0).$
    Thus $\theta_1 / \gamma \to \hat{\theta}_1 = \frac{1}{2 n} \sum_{i=1}^n y_i a_i$ and $\theta_0 / \gamma \to \hat{\theta}_0 = 0$ when $\gamma \to \infty.$ 
    We now use mathematical induction, and assume that $\theta_t / \gamma$ converges to $\hat{\theta}_t$ when $\gamma \to \infty.$ Using simple algebra, we get
    \begin{align*}
        \frac{\theta_{t+1}}{\gamma} 
        &= \frac{\theta_{t} + \gamma \frac{1}{n} \sum_{i=1}^n\frac{1}{1 + \exp(y_i \,a_i^\top \theta_t)} y_i a_i}{\gamma} \\
        &= \frac{\theta_{t}}{\gamma} + \frac{1}{n} \sum_{i=1}^n\frac{1}{1 + \exp(\gamma \cdot y_i \,a_i^\top \frac{\theta_t}{\gamma})} y_i a_i.
    \end{align*}
    For all $t \geq 1,$ notice that $\hat{\theta}_t = \frac{1}{n} \sum_{i=1}^n (0.5 + k_i) y_i a_i$ in \refalgone{eq:batch_perceptron} for some $k_1, \dots, k_n \in \N_0.$ Using Assumption~\ref{ass:bad_dataset}, we have\footnote{This statement is the only and the main reason why we need Assumption~\ref{ass:bad_dataset}. This assumption helps to avoid the corner case when the samples lie on the hyperplane.} $y_i a_i^\top \hat{\theta}_t \neq 0$ for all $i \in [n].$ Since $\theta_t / \gamma \to \hat{\theta}_t$ when $\gamma \to \infty,$ we get $\textnormal{sign} (y_i a_i^\top \theta_t / \gamma) = \textnormal{sign} (y_i a_i^\top \hat{\theta}_t) \neq 0$  for all $i \in [n]$ and $\gamma$ large enough. Therefore $(1 + \exp(\gamma \cdot y_i \,a_i^\top \frac{\theta_t}{\gamma}))^{-1} \to 1$ if $y_i \,a_i^\top \hat{\theta}_t < 0,$ and $(1 + \exp(\gamma \cdot y_i \,a_i^\top \frac{\theta_t}{\gamma}))^{-1} \to 0$ if $y_i \,a_i^\top \hat{\theta}_t > 0$  when $\gamma \to \infty $ for all $i \in [n],$ meaning
        $\frac{\theta_{t+1}}{\gamma} 
        \overset{\gamma \to \infty}{=} \hat{\theta}_t + \frac{1}{n} \sum_{i \in S_t}  y_i a_i.$
    We have showed that $\theta_{t+1} / \gamma \to \hat{\theta}_{t+1}.$
\end{proof}

Thus, indeed, \refalgone{eq:gd} reduces to \refalgone{eq:batch_perceptron} when $\gamma \to \infty.$ It is left to recall the following classical result, which we prove in Section~\ref{sec:theorem:perceptron} for completeness.

\begin{figure*}[t]
\centering
\begin{subfigure}{0.4\textwidth}
    \centering
    \includegraphics[width=\textwidth]{./results_2024/gd_two_layer_eos_linear_one_no_bias_loss_bce_logits_num_samples_5000_more_iters_more_iters_filter_classes_0_1_repeat_accuracy.pdf}
\end{subfigure}
\begin{subfigure}{0.4\textwidth}
    \centering
    \includegraphics[width=\textwidth]{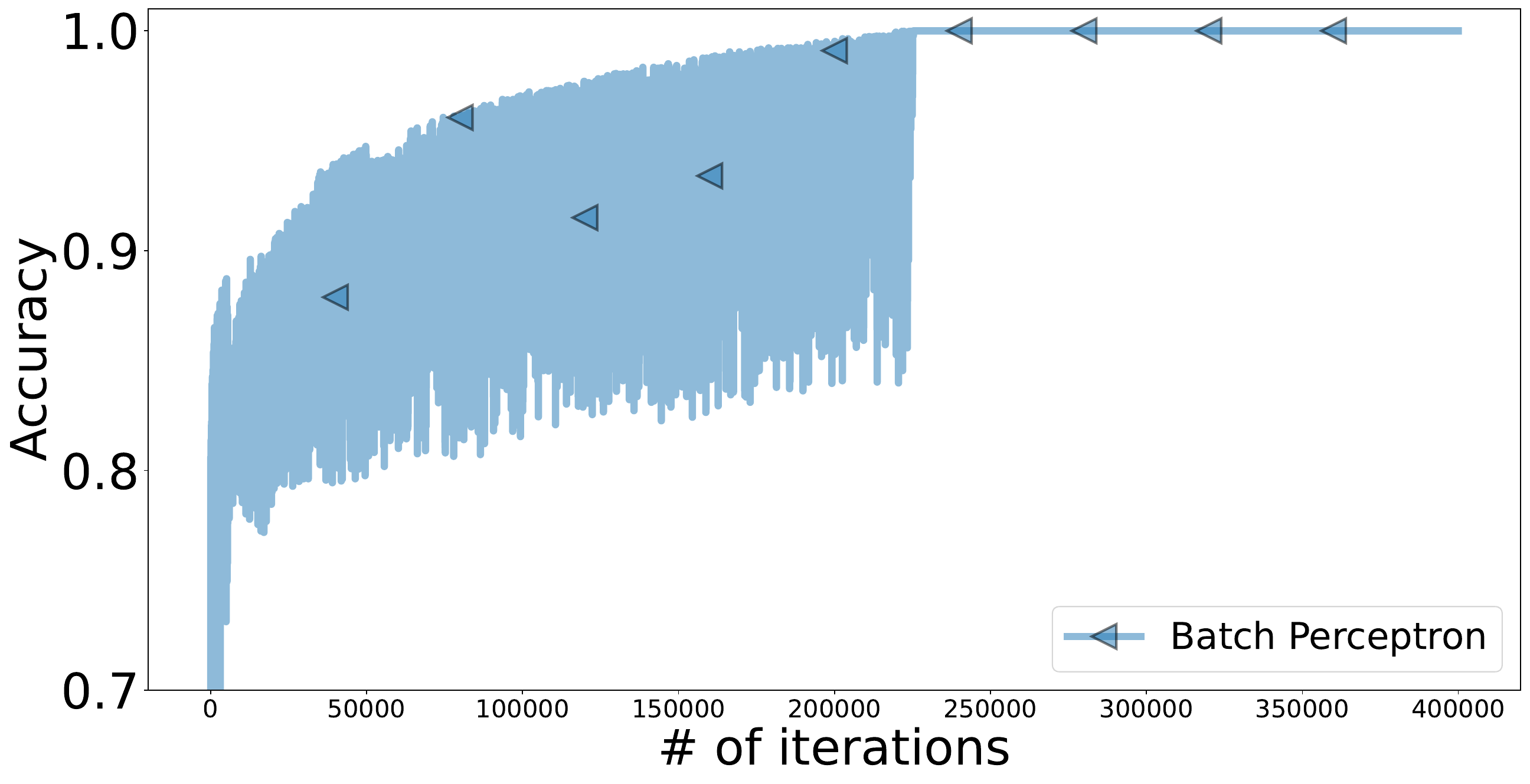}
\end{subfigure}
\caption{We show that \protect\refalgone{eq:gd} with large step sizes aligns with \protect\refalgone{eq:batch_perceptron} on \emph{CIFAR-10}.}
\label{fig:cifar10}
\centering
\begin{subfigure}[t]{0.4\textwidth}
    \centering
    \includegraphics[width=\textwidth]{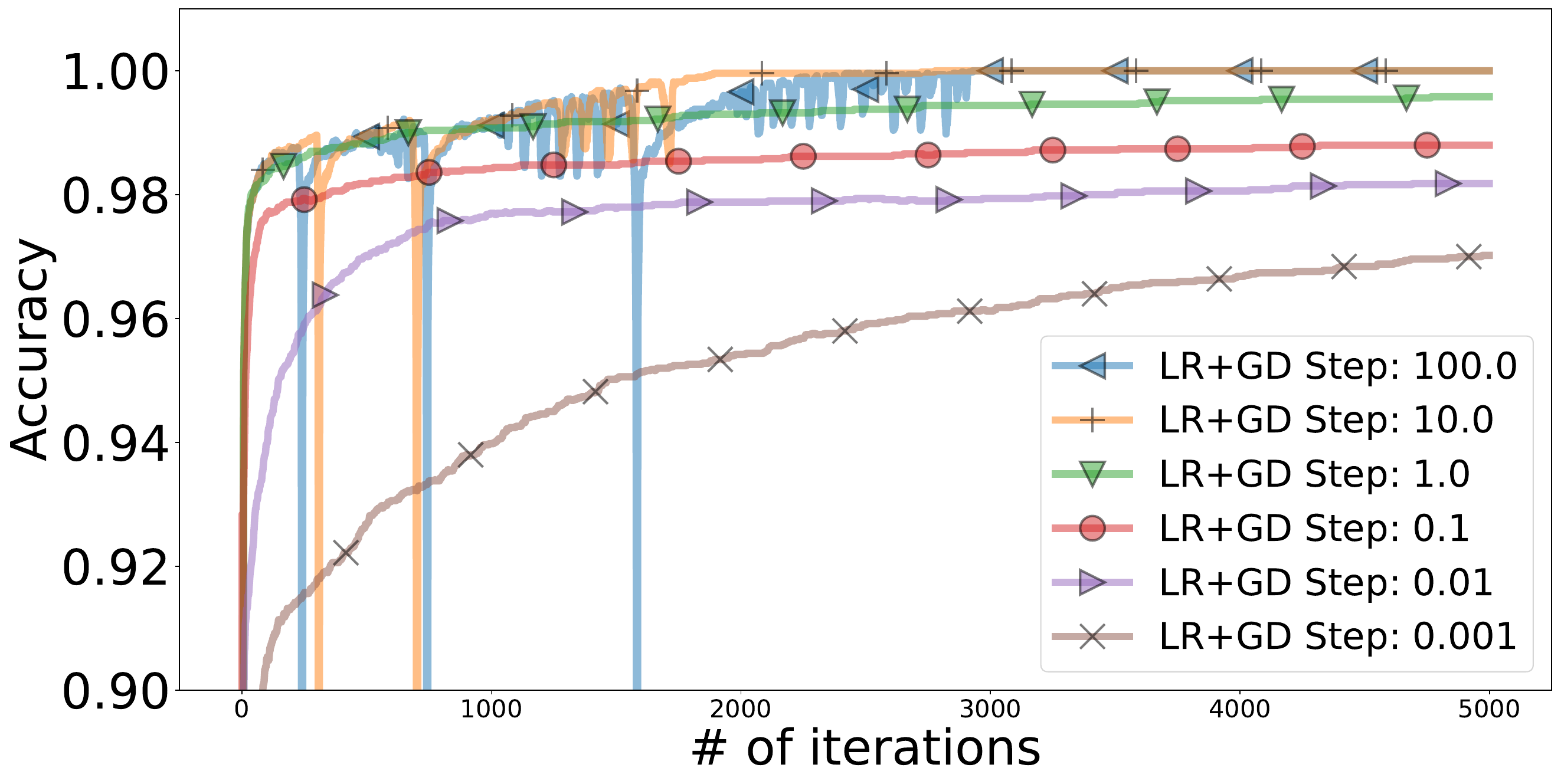}
\end{subfigure}
\begin{subfigure}[t]{0.4\textwidth}
    \centering
    \includegraphics[width=\textwidth]{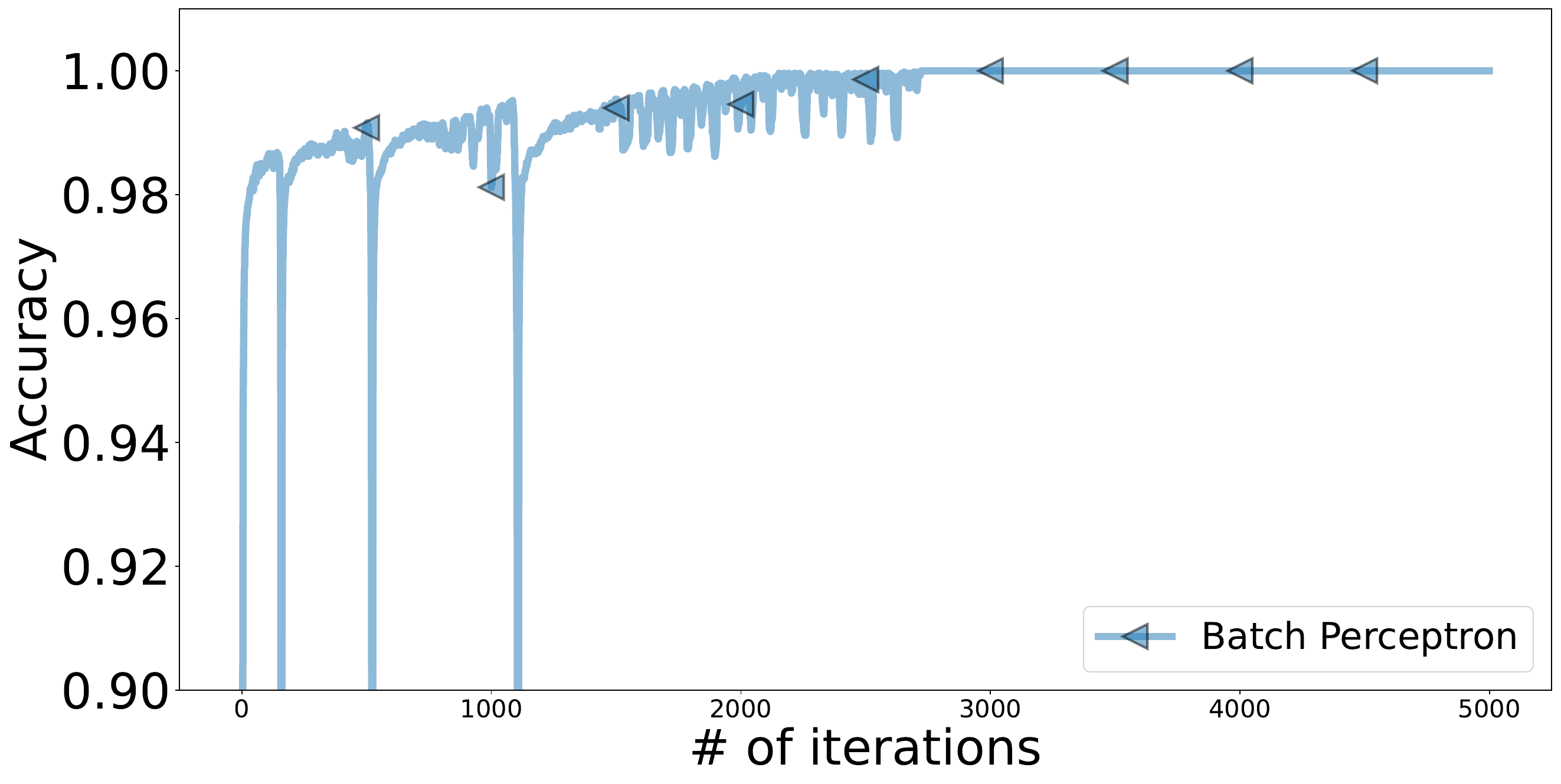}
\end{subfigure}
\caption{We show that \protect\refalgone{eq:gd} with large step sizes aligns with \protect\refalgone{eq:batch_perceptron} on \emph{FashionMNIST}.}
\label{fig:fashionmnist}
\end{figure*}
\begin{restatable}{theorem}{BATCHPERCEPTRON}[\citep{novikoff1962convergence,pattern}]
    \label{theorem:perceptron}
    Let Assumption~\ref{ass:separable} hold. The batch perceptron algorithm (\refalgone{eq:batch_perceptron}) solves \eqref{eq:linear_task} after at most 
    \begin{align}
        \label{eq:bAvqbGayoUiJYuIN}
        \frac{n R^2}{\mu^2}
    \end{align}
    iterations if $\hat{\theta}_{0} = 0.$
\end{restatable}

Theorem~\ref{theorem:perceptron} and Theorem~\ref{theorem:reduction} explain why \refalgone{eq:gd} solves \eqref{eq:linear_task} with $\gamma \to \infty$. 
Notice that the convergence rate \eqref{eq:bAvqbGayoUiJYuIN} does not degenerate when $\gamma \to \infty.$ 
Crucially, we provide the convergence guarantees for the task \eqref{eq:linear_task}, not for the task \eqref{eq:logistic_regression}. The latter is merely \emph{a proxy problem}. In practice, what matters is how fast we find a separator rather than how fast the loss converges to zero, and in fact, we will see in Section~\ref{sec:unrealiable} that the logistic loss is an unreliable metric. \\
\emph{Remark:} A scaled version of \eqref{eq:logistic_regression}, $\frac{1}{t \times n} \sum_{i=1}^{n} \log\left(1 + \exp(-t \times y_i a_i^\top \theta)\right),$ reduces to the perceptron loss when $t \to \infty.$ Thus, there can potential connection between \ref{eq:gd} with large step sizes and this fact.

\textbf{Numerical experiments.} We now numerically verify the obtained results by comparing the performance of two algorithms that solve \eqref{eq:linear_task}: logistic regression with gradient descent (\refalgone{eq:gd}) and the perceptron algorithm (\refalgone{eq:batch_perceptron}). \refalgone{eq:batch_perceptron} has no hyperparameters, while \refalgone{eq:gd} requires the step size $\gamma$. We evaluate these algorithms on four datasets: \emph{CIFAR-10} \citep{krizhevsky2009learning}, \emph{FashionMNIST} \citep{xiao2017fashion}, \emph{EuroSAT} \citep{helber2019eurosat}, and \emph{MNIST} \citep{lecun2010mnist}, selecting two classes and $5\,000$ samples from each dataset (see details in Section~\ref{sec:exp_setup}). For \refalgone{eq:gd}, we vary the step size from $0.001$ to $100$. 
Figures~\ref{fig:cifar10}, \ref{fig:fashionmnist}, \ref{fig:eurosat}, and \ref{fig:mnist} present the results side by side. The results indicate that \refalgone{eq:gd} with a small step size has monotonic and stable convergence curves. However, as the step size increases, the plots become unstable and chaotic. The behavior of \refalgone{eq:gd} with large step sizes aligns closely with that of \refalgone{eq:batch_perceptron} across all datasets almost exactly, which supports our theory. And in the limit of $\gamma \to \infty,$ converges to \refalgone{eq:batch_perceptron}. We also run experiments with $1\,000$ and $10\,000$ samples in Section~\ref{sec:more_data_points} for additional support.

\section{Logistic Loss and the Norm of Gradients are Unreliable Metrics}
\label{sec:unrealiable}

Looking closer at the results of the experiments on datasets (Figures~\ref{fig:cifar10_acc}, \ref{fig:fashionmnist_acc}, \ref{fig:eurosat_more}, and \ref{fig:mnist_more}), we notice that the large step size not only leads to faster convergence rates but also to larger function values (before the moment when Accuracy~$=1.0$). Is this a coincidence, or is there some pattern? We can prove the following simple theorem that explains the phenomenon:

\begin{figure*}
\centering
\begin{subfigure}[t]{0.33\textwidth}
    \centering
    \includegraphics[width=\textwidth]{./results_2024/gd_two_layer_eos_linear_one_no_bias_loss_bce_logits_num_samples_5000_more_iters_more_iters_filter_classes_0_1_repeat_accuracy.pdf}
\end{subfigure}
\begin{subfigure}[t]{0.33\textwidth}
    \centering
    \includegraphics[width=\textwidth]{./results_2024/gd_two_layer_eos_linear_one_no_bias_loss_bce_logits_num_samples_5000_more_iters_more_iters_filter_classes_0_1_repeat_loss.pdf}
\end{subfigure}
\begin{subfigure}[t]{0.33\textwidth}
    \centering
    \includegraphics[width=\textwidth]{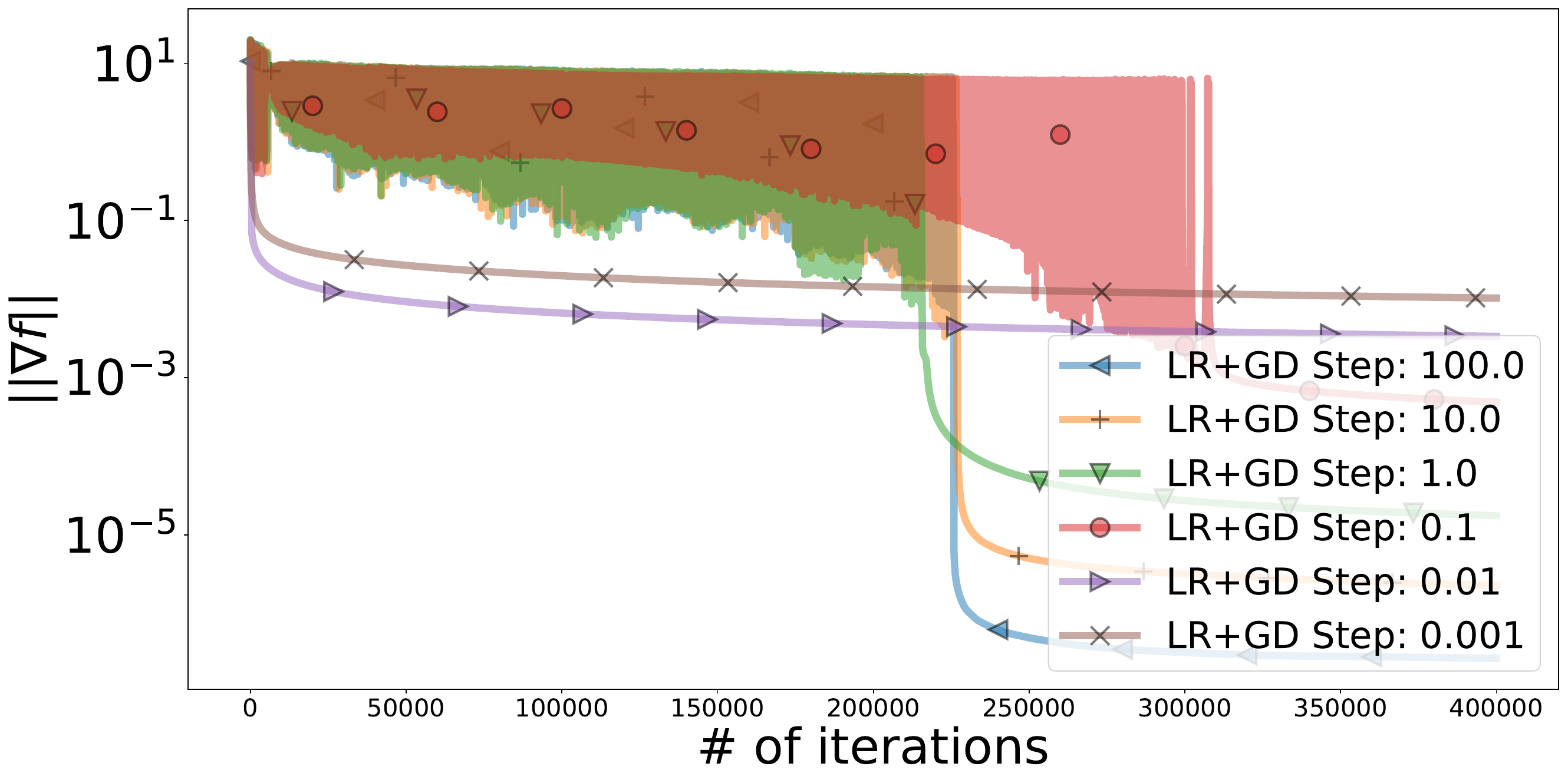}
\end{subfigure}
\caption{Accuracy, function values, and the norm of gradients of the logistic loss \eqref{eq:logistic_regression} on \emph{CIFAR-10} during the runs of \protect\refalgone{eq:gd}.}
\label{fig:cifar10_acc}
\centering
\begin{subfigure}[t]{0.33\textwidth}
    \centering
    \includegraphics[width=\textwidth]{./results_2024/gd_two_layer_eos_linear_one_no_bias_loss_bce_logits_fashion_mnist_num_samples_5000_more_iters_more_iters_filter_classes_0_4_accuracy.pdf}
\end{subfigure}
\begin{subfigure}[t]{0.33\textwidth}
    \centering
    \includegraphics[width=\textwidth]{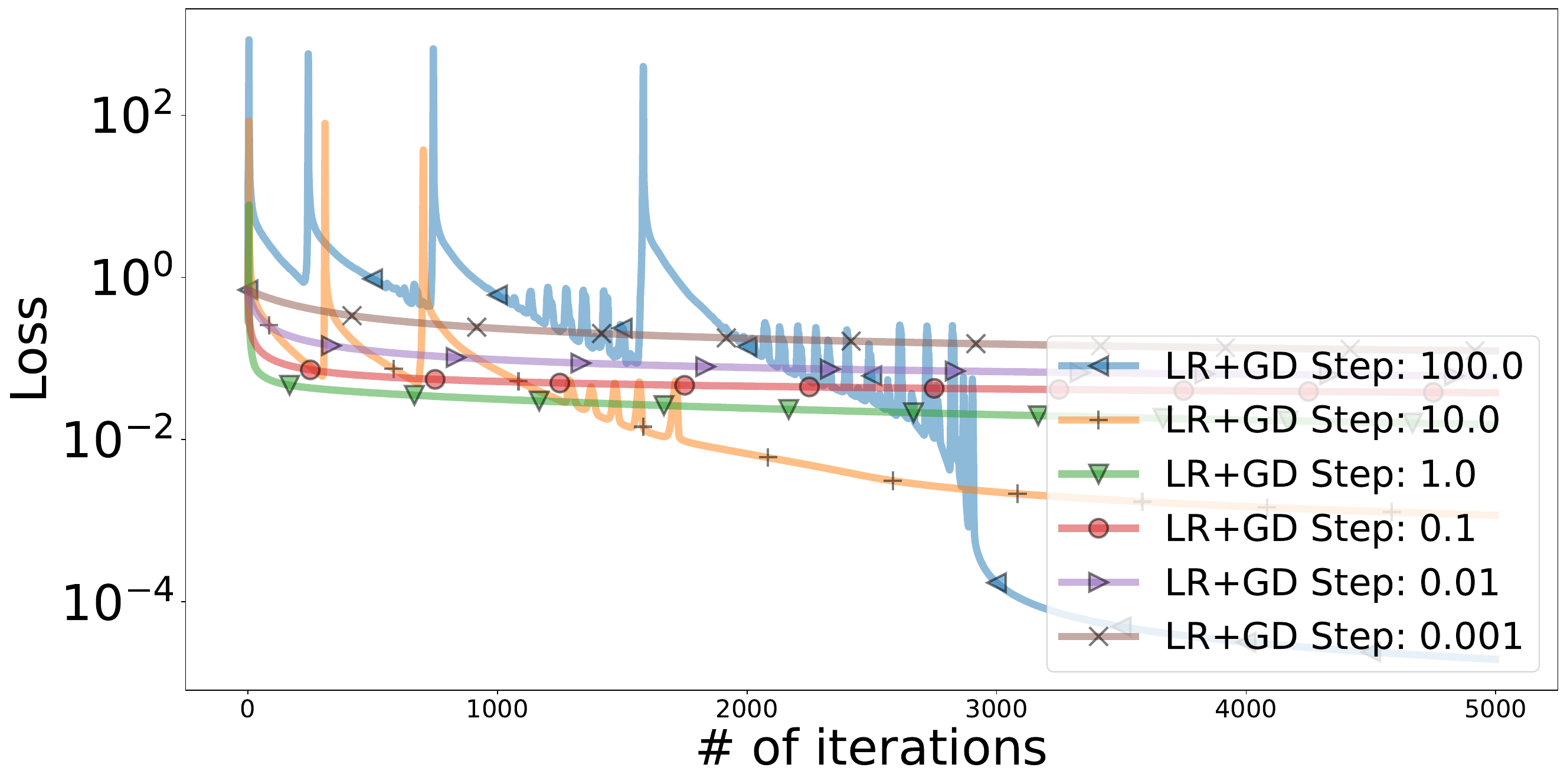}
\end{subfigure}
\begin{subfigure}[t]{0.33\textwidth}
    \centering
    \includegraphics[width=\textwidth]{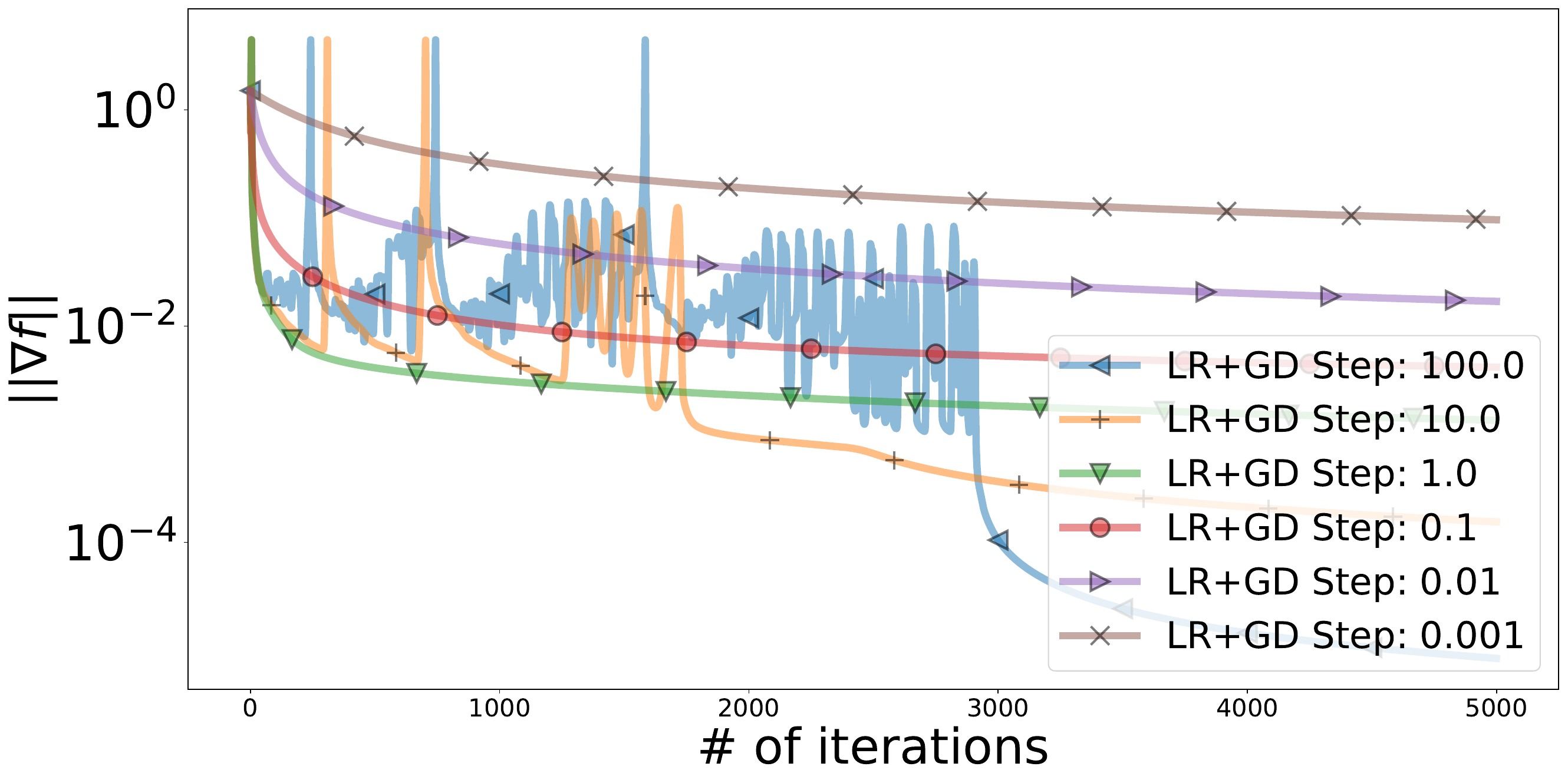}
\end{subfigure}
\caption{Accuracy, function values, and the norm of gradients on \emph{FashionMNIST}. The larger the step size in \protect\refalgone{eq:gd}, the faster \protect\refalgone{eq:gd} solves \eqref{eq:linear_task} and gets Accuracy$=1.0.$ At the same time, the larger the step size, the higher the loss values.}
\label{fig:fashionmnist_acc}
\end{figure*}

\begin{theorem}
    Assume that $\theta_1 = 0.$ There exists a separable dataset (Assumption~\ref{ass:separable}) such that 
    \begin{enumerate}
        \item $f(\theta_1) \to \infty$ and $f(\theta_2) \to 0$ when $\gamma \to \infty,$
        \item $\nicefrac{\theta_2}{\gamma}$ is a solution of \eqref{eq:linear_task} when $\gamma \to \infty,$
        \item $\norm{\nabla f(\theta_1)} \to \nicefrac{\sqrt{2}}{2}$ and $\norm{\nabla f(\theta_2)} \to 0$ when $\gamma \to \infty,$
    \end{enumerate}
    where $\theta_1$ and $\theta_2$ are the first and second iterates of \ref{eq:gd}.
\end{theorem}

\begin{proof}
    We take the dataset with one sample $(1, -1)^\top$ assigned to the class $1$ and one sample $(-1, -4)^\top$ assigned to the class $-1.$ Using \eqref{eq:gd} and \eqref{eq:grad}, we have $\theta_{1} = \gamma (\frac{2}{4}, \frac{3}{4})^\top.$ Thus $f(\theta_1) = \frac{1}{2} \left(\log\left(1 + \exp(\frac{\gamma}{4})\right) + \log\left(1 + \exp(-\frac{7 \gamma}{2})\right)\right),$ meaning $f(\theta_1) \to \infty$ when $\gamma \to \infty.$ On the other hand, a direct calculation yield
    \begin{align*}
        \textstyle \frac{\theta_2}{\gamma} &= \textstyle (\frac{2}{4}, \frac{3}{4})^\top + \frac{1}{2} \left(\left(1 + \exp(-\frac{\gamma}{4})\right)^{-1} (1, -1)^\top \right.\\
        &\textstyle \left.+ \left(1 + \exp(\frac{7 \gamma}{2})\right)^{-1} (1, 4)^\top\right) \to (1, \frac{1}{4})^\top
    \end{align*}
    when $\gamma \to \infty.$ The point $(1, \frac{1}{4})^\top$ is a solution of \eqref{eq:linear_task}, and $f(\theta_2) \to 0$ when $\gamma \to \infty.$
    The last statement of the theorem can be verified using \eqref{eq:grad}.
\end{proof}

Even though the first value $f(\theta_1)$ of the loss indicates divergence, the algorithm solves \eqref{eq:linear_task} after two steps when $\gamma \to \infty$! In this example, it is clear that the high value of $f(\theta_1)$ does not reflect the fact that the algorithm will solve the problem in the next step. The experiments from 
Figures~\ref{fig:cifar10_acc} and \ref{fig:fashionmnist_acc} (see also Section~\ref{sec:exp_setup}) support this theorem. That also applies to the norm of gradients. The ratio between $\norm{\nabla f(\theta_1)}$ and $\norm{\nabla f(\theta_2)}$ can be arbitrarily large for large $\gamma.$ In Figures~\ref{fig:cifar10_acc} and \ref{fig:fashionmnist_acc} (see also Section~\ref{sec:exp_setup}), the norm of gradients are chaotic and large until the moment when \refalgone{eq:gd} finds a solution of \eqref{eq:linear_task}.

\section{LR+GD is a Suboptimal Method}

In the previous section, we explain that logistic loss and the norm of gradient do not provide sufficient information about our proximity to solving \eqref{eq:linear_task}. Recall that GD is a method of choice because, for instance, it is an optimal method in the nonconvex setting \citep{carmon2020lower} and has the optimal convergence rate by the norm of gradients. In the case of the task \eqref{eq:linear_task} and \refalgone{eq:gd}, this is no longer true.
Indeed, let us now consider the iteration rate $\nicefrac{n R^2}{\mu^2}$ from Theorem~\ref{theorem:perceptron} by \refalgone{eq:gd} when $\gamma \to \infty$. 
The iteration rate is suboptimal since it linearly depends on $n,$ and can be improved by the classical (non-batch) perceptron algorithm \citep{novikoff1962convergence}.
The following lower bound proves that the dependence is unavoidable for \refalgone{eq:batch_perceptron}.
\begin{theorem}
    \label{sec:lower_bound}
    There exists a separable dataset (Assumption~\ref{ass:separable}) with $\mu = \Theta(1)$ and $R = \Theta(1)$ such that \refalgone{eq:batch_perceptron} (\refalgone{eq:gd} when $\gamma \to \infty$) requires at least $\Omega(n)$ iterations to solve \eqref{eq:linear_task} if $\hat{\theta}_{0} = 0$ and $n \geq 10.$
\end{theorem}
\begin{proof}
    We take the dataset with one sample $(0.5, -1)^\top$ assigned to the class $1$ and $n - 1$ samples $(-0.5, -1)^\top$ assigned to the class $-1.$ We start at the point $\hat{\theta}_0 = (0, 0)^\top.$ 
    Then $\hat{\theta}_{1} = \hat{\theta}_{0} + \frac{1}{2 n}\sum_{i = 1}^n y_i a_i = (0.25, \frac{n - 2}{2 n})^\top.$ 
    Only the sample from the class $-1$ is misclassified at $\hat{\theta}_{1}$ and belongs to $S_1.$ Therefore $\hat{\theta}_{2} = (0.25, \frac{n - 2}{2 n})^\top + \frac{1}{n} (0.5, -1)^\top = (0.25 (1 + \frac{2}{n}), \frac{n - 4}{2 n})^\top.$ Again, only the sample from the class $-1$ belongs to $S_2.$ Thus $\hat{\theta}_{3} = (0.25 (1 + \frac{2}{n}), \frac{n - 4}{2 n})^\top + \frac{1}{n} (0.5, -1)^\top = (0.25 (1 + \frac{4}{n}), \frac{n - 6}{2 n})^\top.$ This will happen further with $\hat{\theta}_{4}, \dots, \hat{\theta}_{k}$ until either the last coordinate becomes negative (the samples from the class $1$ will be misclassified), i.e., $\frac{n - 2 k}{2 n} < 0,$ or the sample from the class $-1$ stops being misclassified, i.e., $0.5 \times 0.25 (1 + \frac{2 k - 2}{n}) + -1 \times \frac{n - 2 k}{2 n} > 0.$ Both conditions require $k$ to be greater or equal to $\Omega(n).$ 
\end{proof}

\begin{figure*}[t]
    \centering
    \begin{subfigure}[t]{0.33\textwidth}
        \centering
        \includegraphics[width=\textwidth]{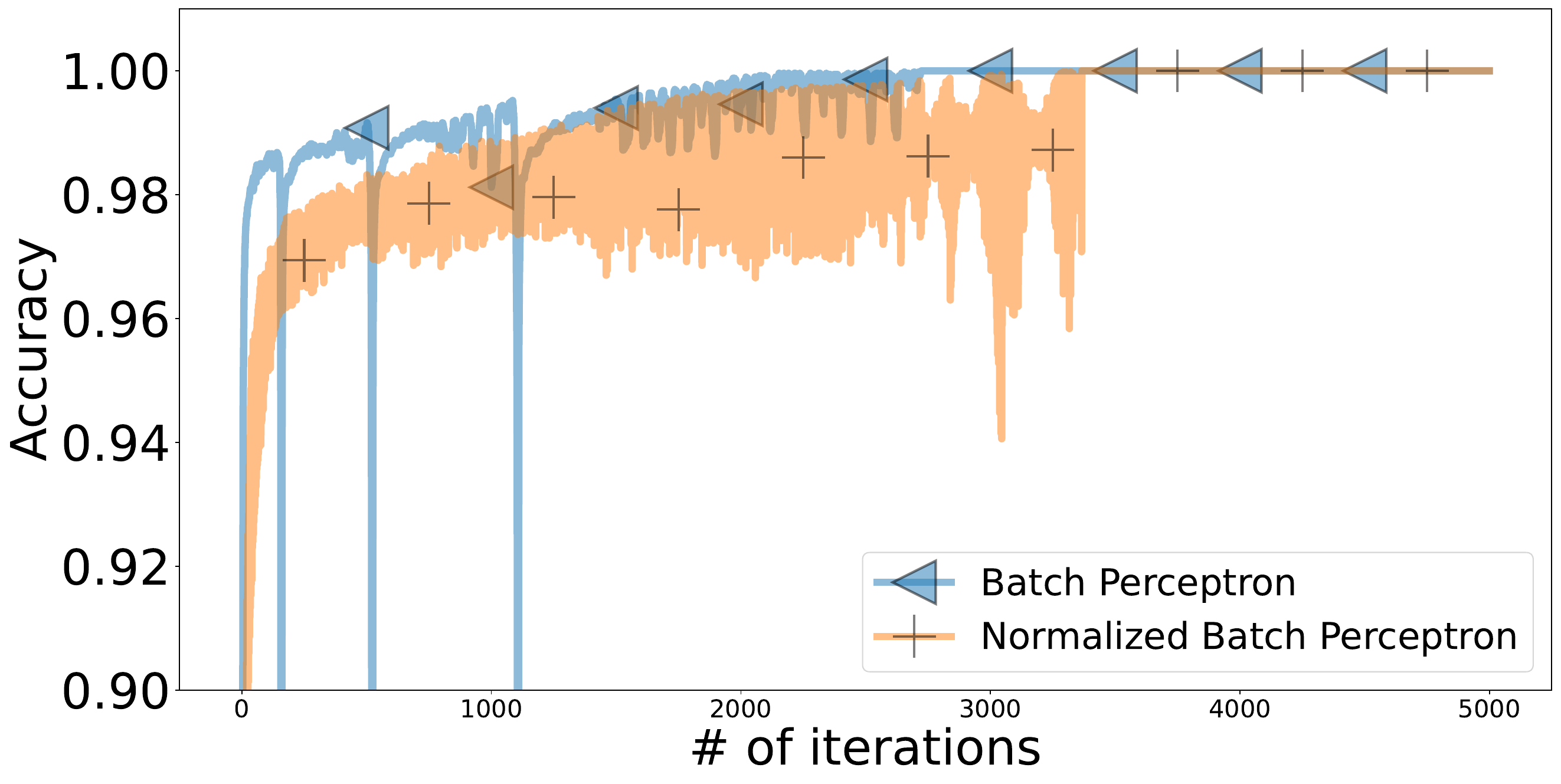}
        \caption{\emph{Fashion MNIST}}
    \end{subfigure}
    \begin{subfigure}[t]{0.33\textwidth}
        \centering
        \includegraphics[width=\textwidth]{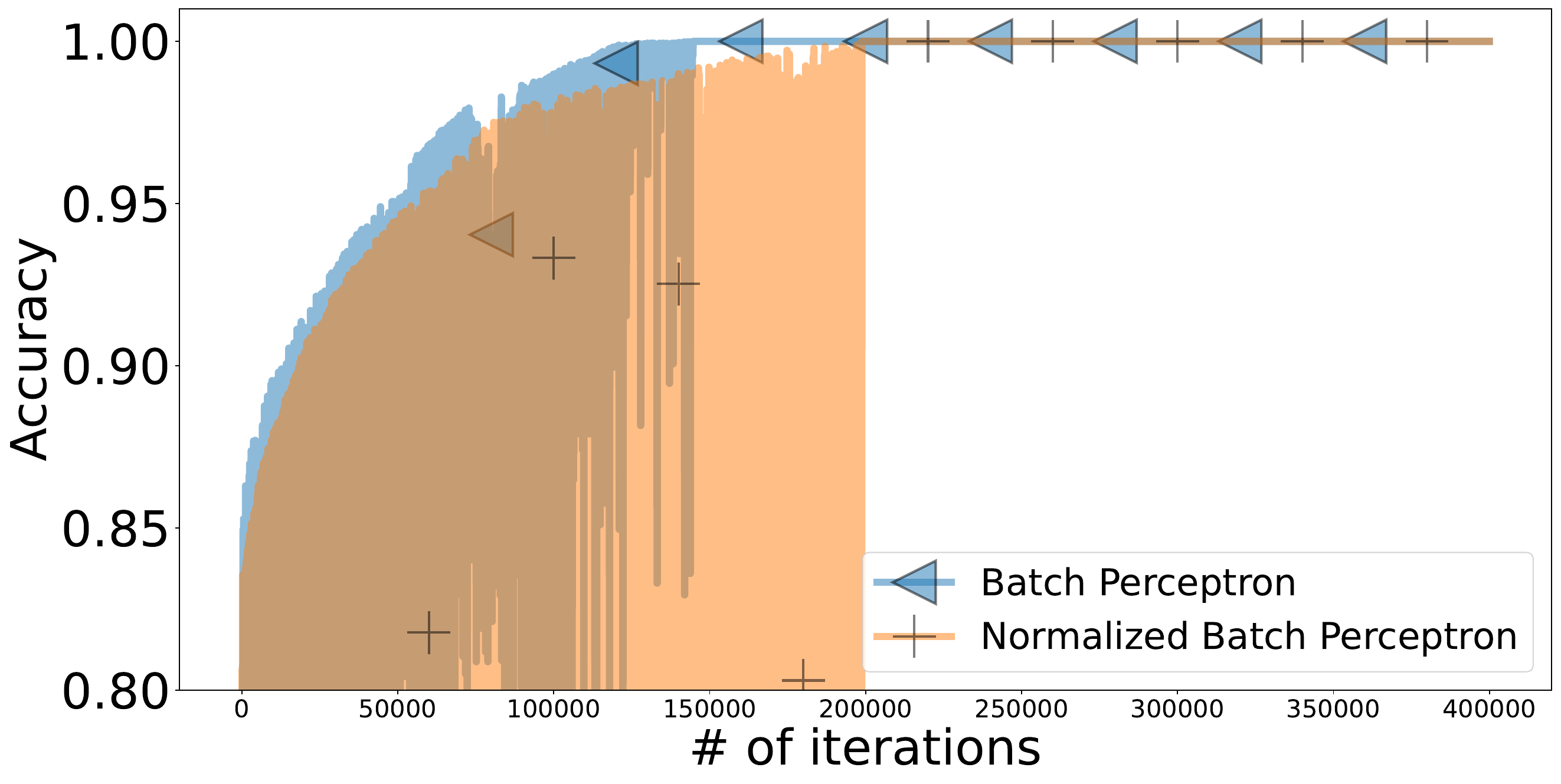}
        \caption{\emph{EuroSAT}}
    \end{subfigure}
    \begin{subfigure}[t]{0.33\textwidth}
        \centering
        \includegraphics[width=\textwidth]{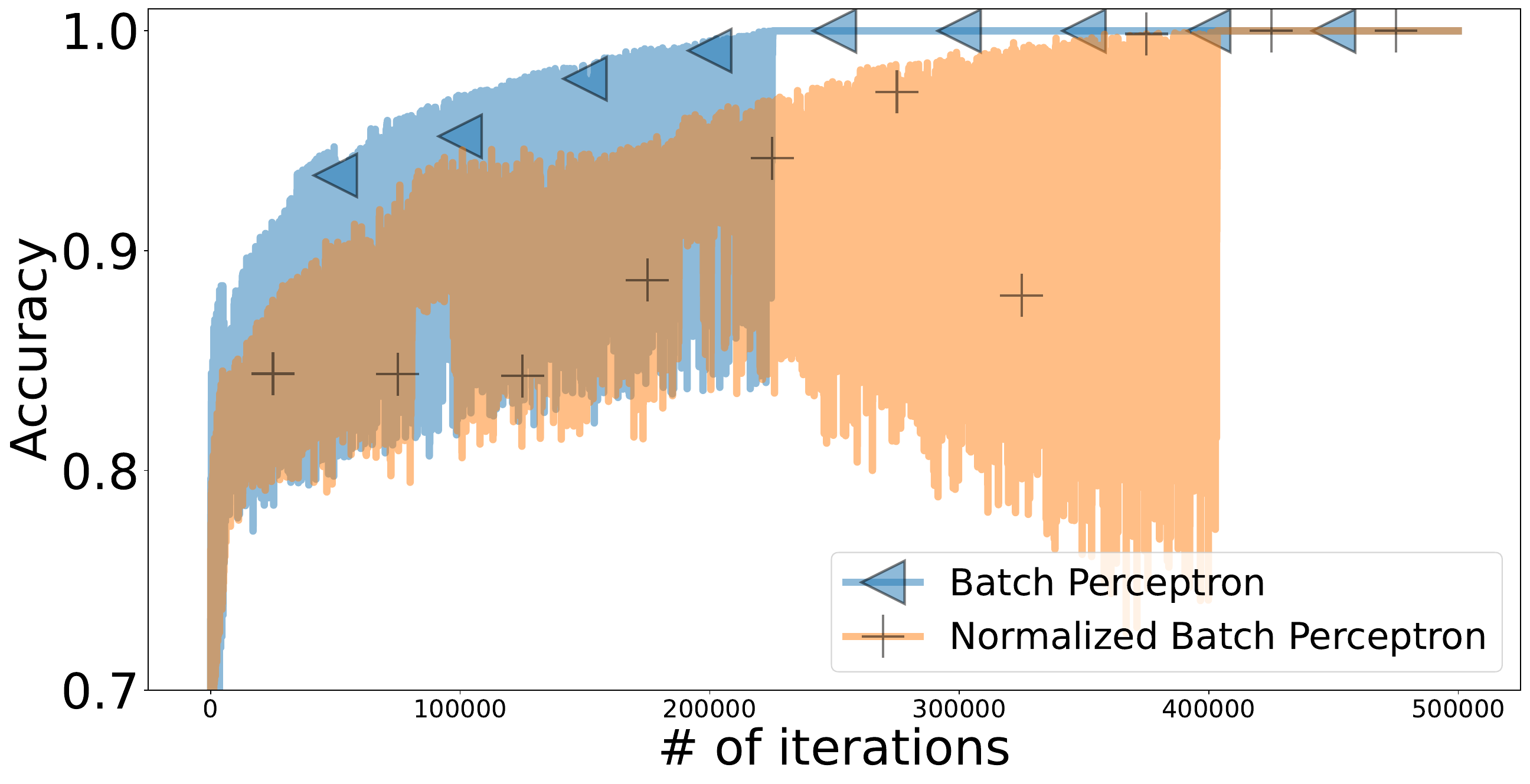}
        \caption{\emph{CIFAR-10}}
    \end{subfigure}
    \caption{Comparison of perceptron algorithms.}
    \label{fig:fashion_mnist_compare_perceptron}
    \centering
    \begin{subfigure}[t]{0.33\textwidth}
        \centering
        \includegraphics[width=\textwidth]{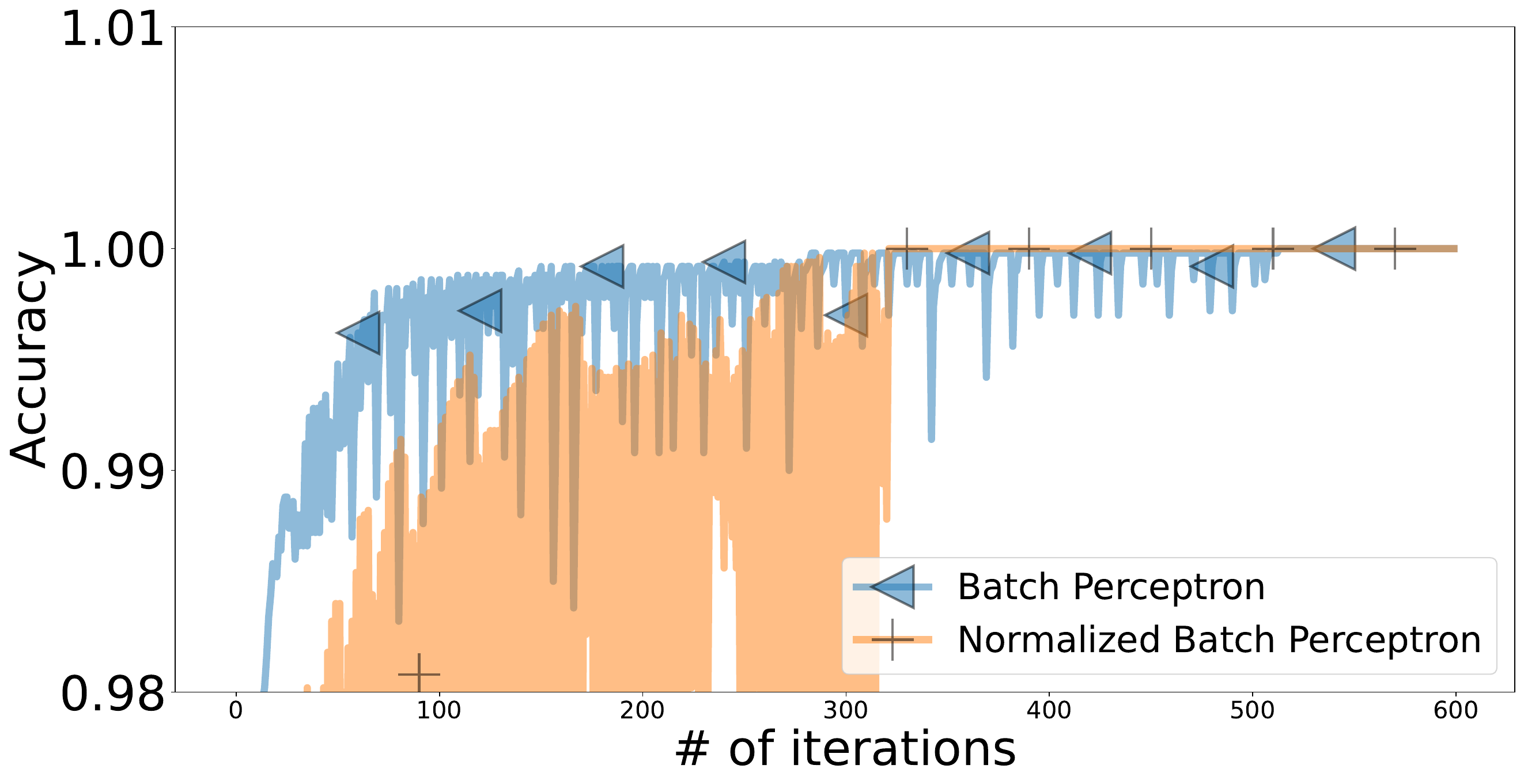}
        \caption{\emph{Fashion MNIST}}
    \end{subfigure}
    \begin{subfigure}[t]{0.33\textwidth}
        \centering
        \includegraphics[width=\textwidth]{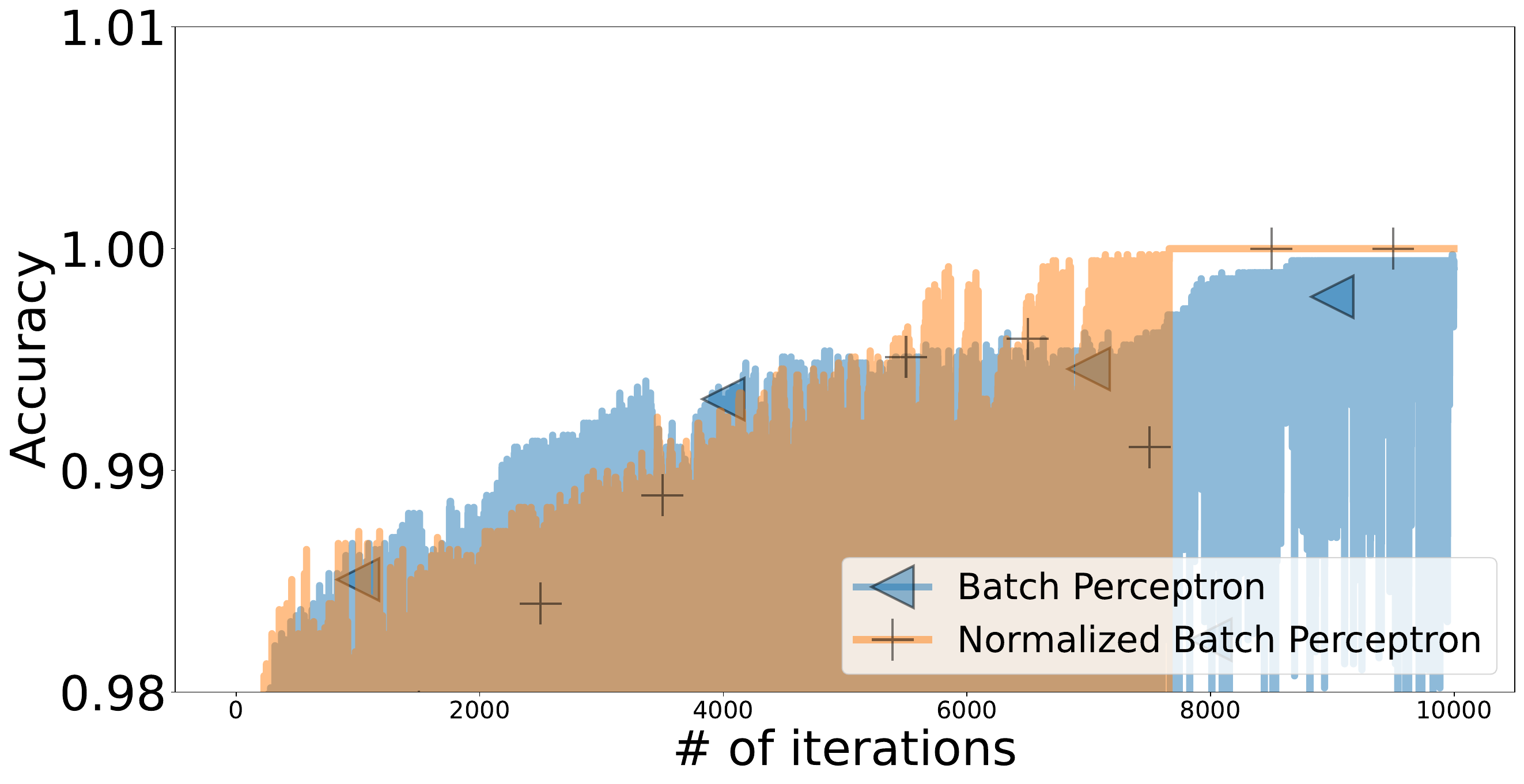}
        \caption{\emph{EuroSAT}}
    \end{subfigure}
    \begin{subfigure}[t]{0.33\textwidth}
        \centering
        \includegraphics[width=\textwidth]{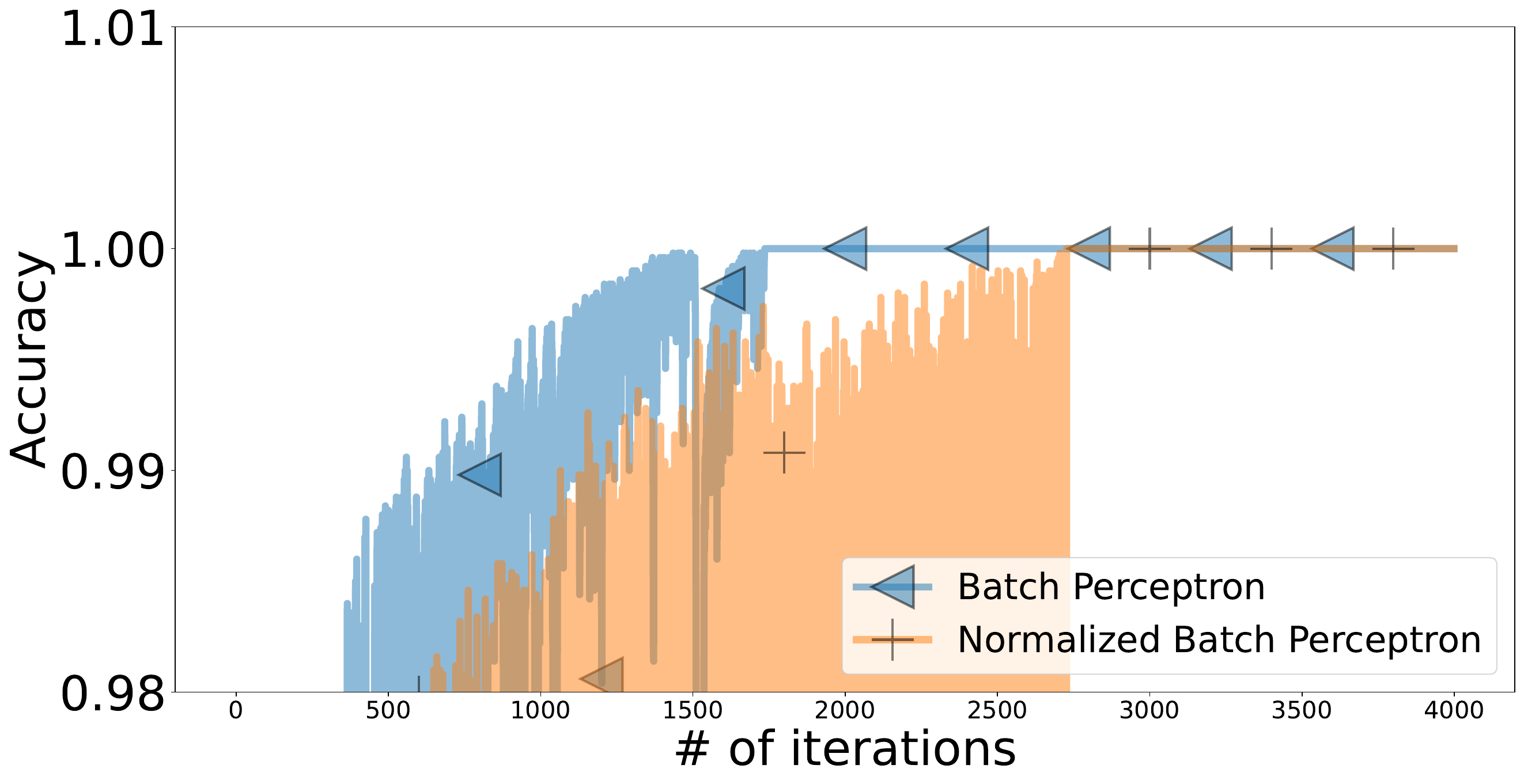}
        \caption{\emph{CIFAR-10}}
    \end{subfigure}
    \caption{Comparison of perceptron algorithms on \textbf{imbalanced} data (see details in Section~\ref{sec:exp_setup_imbalanced}). On two of three datasets \protect\refalgone{eq:norm_perceptron} converges to Accuracy=$1.0$ faster than \protect\refalgone{eq:batch_perceptron}.}
    \label{fig:fashion_mnist_compare_perceptron_imbalanced}
\end{figure*}

We have proved that \refalgone{eq:gd} with $\gamma \to \infty$ is a suboptimal method. At the same time, it is well-known that we can improve the rate using the classical versions of the perceptron algorithm that yield better iteration rates:
\begin{theorem}[\citep{pattern},Theorem 5.1]
    The classical perceptron algorithm \citep{novikoff1962convergence}, defined as  
    \begin{equation}
        \label{eq:perceptron}\tag{Perceptron}
    \begin{aligned}
        &\textnormal{For all $t \geq 0,$ find the set } S_t \eqdef \{i \in [n]\,:y_i \,{a_i}^\top \hat{\theta}_t \leq 0\,\}, \\
        & \textnormal{choose }j \in S_t \textnormal{ and take the step } \hat{\theta}_{t+1} = \hat{\theta}_{t} + y_j a_j,
    \end{aligned}
    \end{equation}
    solves \eqref{eq:linear_task} after at most 
    $\frac{R^2}{\mu^2}$ 
    iterations if $\hat{\theta}_{0} = 0.$
\end{theorem}
We can also consider a practical variant with proper normalization. Using a different normalization factor, $\nicefrac{1}{\abs{S_t}}$ instead of $\nicefrac{1}{n},$ we can provide better guarantees:
\begin{restatable}{theorem}{NORMALIZEDPERCERTRON}[Proof in Section~\ref{sec:normalized_perceptron}]
    \label{theorem:normalized_perceptron}
    The batch perceptron algorithm with a proper normalization, defined as  
    \begin{equation}
        \label{eq:norm_perceptron}\tag{Normalized Batch Perceptron}
    \begin{aligned}
        &\textnormal{For all $t \geq 0,$ find the set } S_t \eqdef \{i \in [n]\,:y_i \,{a_i}^\top \hat{\theta}_t \leq 0\,\},\\
        &\textnormal{and take the step } \hat{\theta}_{t+1} = \hat{\theta}_{t} + \frac{1}{{\abs{S_t}}}\sum_{i \in S_t} y_i a_i \textnormal{ while $\abs{S_t} \neq 0,$}
    \end{aligned}
    \end{equation}
    solves \eqref{eq:linear_task} after at most 
    \begin{align*}
        \frac{R^2}{\mu^2}
    \end{align*}
    iterations if $\hat{\theta}_{0} = 0.$
\end{restatable}

One can see that \refalgone{eq:perceptron} and \refalgone{eq:norm_perceptron} have $n$ times better convergence rates than \refalgone{eq:batch_perceptron}. The only difference between \refalgone{eq:norm_perceptron} and \refalgone{eq:batch_perceptron} is the proper normalization, which is crucial to get a better iteration rate. \\
\textbf{Numerical experiments.} In Figure~\ref{fig:fashion_mnist_compare_perceptron}, we compare \refalgone{eq:norm_perceptron} and \refalgone{eq:batch_perceptron} numerically and observe that \refalgone{eq:batch_perceptron} converges slightly better in practice despite worse theoretical guarantees.
However, in a setup with imbalanced data, described in Section~\ref{sec:exp_setup_imbalanced}, we observe that \refalgone{eq:norm_perceptron} finds a solution faster than \refalgone{eq:batch_perceptron} (Figure~\ref{fig:fashion_mnist_compare_perceptron_imbalanced}). One research question is to uncover the reasons behind this. A potential high-level explanation is that \refalgone{eq:batch_perceptron} performs well ``on average'', but not robust to inbalanced data.

\begin{table*}[t]
    \centering
    \begin{minipage}{.49\linewidth}
    \caption*{\textbf{The ``worst-case'' dataset from Theorem~\ref{sec:lower_bound}}}
      \centering
      \input{./results_2024/plot_local_optimization_2024_paper_new_algorithm_2_table.tex}
    \end{minipage}%
    \begin{minipage}{.49\linewidth}
        \centering
        \caption*{\textbf{Imbalanced \emph{Fashion MNIST}} (see Sec.~\ref{sec:exp_setup_imbalanced})}
        \input{./results_2024/plot_local_optimization_2024_paper_new_algorithm_1_table.tex}
    \end{minipage}%
    \\\vspace{0.5cm}
    \begin{minipage}{.49\linewidth}
        \centering
        \caption*{\textbf{\emph{Fashion MNIST}} (default setup from Sec.~\ref{sec:exp_setup})}
        \input{./results_2024/plot_local_optimization_2024_paper_new_algorithm_3_table.tex}
    \end{minipage}%
    \caption{The \# of iterations required by \ref{eq:gd} and \ref{eq:norm_gd} to solve \eqref{eq:linear_task} with various step sizes. Notice that \protect\refalgone{eq:norm_gd} solves \eqref{eq:linear_task} \emph{150 times} faster than \protect\refalgone{eq:gd} on the first dataset, 1.7 times faster on the second dataset, and only 1.1 times slower on the third dataset.}
    \label{eq:tbl}
\end{table*}

\section{A New Method, Normalized LR+GD, Yields a Faster Iteration Rate}

Since \refalgone{eq:norm_perceptron} converges faster than \refalgone{eq:batch_perceptron} by $n$ times, it raises the question of whether it is possible to modify \refalgone{eq:gd} and obtain a better rate. The answer is affirmative. 
We propose the following method:
\begin{align}
    \label{eq:norm_gd}\tag{Normalized LR+GD}
    \theta_{t+1} = \theta_{t} - \gamma \beta_t \nabla f(\theta_t),
\end{align}
where $\nabla f(\theta_t)$ is the gradient of \eqref{eq:logistic_regression} at the point $\theta_t,$ and 
\begin{align*}
    \textstyle \beta_t = \left(\frac{1}{n}\sum_{i=1}^n (1 + \exp(y_i a_i^\top \theta_t))^{-1}\right)^{-1}.
\end{align*}
We reverse-engineered this method from \refalgone{eq:norm_perceptron}, observing that $\nabla f(\theta_t) \to \frac{1}{\abs{S_t}}\sum_{i \in S_t} y_i a_i$ (see Theorem~\ref{theorem:reduction}) and $\beta_t \to \abs{S_t}$ when $\gamma \to \infty.$ There are many ways how one can interpret it. For instance, it can be seen as \refalgone{eq:gd} but with adaptive step sizes. Note that this method is specialized for the problem \eqref{eq:logistic_regression} because $\beta_t$ requires the features $\{a_i\}_{i=1}^n$ and labels $\{y_i\}_{i=1}^n$. 
We can prove that this method solves \eqref{eq:linear_task} faster than \refalgone{eq:gd}:
\begin{restatable}{theorem}{BATCHPERCEPTRONNONASYMPTFIX}
    \label{theorem:nonasymfix}
    Let Assumption~\ref{ass:separable} hold. \refalgone{eq:norm_gd} solves \eqref{eq:linear_task} after at most 
    \begin{align}
        \label{eq:upper_bound}
        \frac{R^2}{\mu^2} + \frac{2\log(2n - 1)}{\gamma \mu^2}
    \end{align}
    iterations if $\theta_0 = 0.$
\end{restatable}

The theorem suggests that we should increase the step size and let $\gamma \to \infty$. \\
\textbf{Numerical experiments.} This theoretical insight is supported by the experiments from Table~\ref{eq:tbl}, where the best convergence rate is achieved with large step sizes. In Table~\ref{eq:tbl}, we compare the methods and observe that \refalgone{eq:norm_gd} is more robust to imbalanced datasets.

\section{Conclusion}
In this work, we analyze the classification problem \eqref{eq:linear_task} through the logistic regression \eqref{eq:logistic_regression}. Our key takeaways are

\textbf{1.} Logistic regression and GD with large step sizes reduces to the celebrated (batch) perceptron algorithm, which can explain why \refalgone{eq:gd} solves \eqref{eq:linear_task} even when $\gamma \to \infty.$

\textbf{2.} We can not fully trust function and gradient values when optimizing logistic regression problems. The same caution applies to theoretical works. If a theoretical method has a good convergence rate based on function value residuals or gradient norms ($f(\theta_t) - f^* \leq \dots $ or $\norm{\nabla f(\theta_t)}^2 \leq \dots$), it does not necessarily mean that this method will perform well in practical machine learning tasks. In fact, the opposite may be true. 

\textbf{3.} While it is well-known that GD is an \emph{optimal} method for nonconvex problems when measuring convergence by \emph{the norm of gradients} \citep{carmon2020lower}, and (almost\footnote{There exists the accelerated gradient method by \citet{nesterov1983method,nesterov2018lectures}}) optimal for convex problems when measuring convergence by \emph{function values} \citep{nesterov2018lectures}, the iteration rate of GD to solve \eqref{eq:linear_task} with logistic regression is suboptimal and does not scale with \# of of data points $n$ in the worst case. This can be improved with the new \ref{eq:norm_gd} method. 

\section{Future Work}

\textbf{Nonlinear models and neural networks.} The natural question is whether extending this work to nonlinear models and neural networks is possible.
In the general case, we have to 
\begin{align*}
    \textnormal{find a vector $\theta \in \R^m$ such that } y_i g(a_i;\theta) > 0 \quad \forall i \in [n],
\end{align*}
where $g \,:\,\ \R^d \times \R^m \to \R$ is a nonlinear mapping. The mathematical aspects in this section are not strict but rather serve as a foundation for future research since the general case with nonlinear models is very challenging and unexplored. All theorems from the previous sections heavily utilize the fact that $g$ is a linear model. If $g$ is a neural network, then it is well-known that logistic regression will diverge for large step sizes. Let us look at the gradient step with the logistic loss:
\begin{align*}
    \theta_{t+1} = \theta_t + \frac{1}{n} \sum_{i=1}^n (1 + \exp(y_i g(a_i;\theta_t)))^{-1} \nabla_{\theta} g(a_i;\theta_t).
\end{align*}
For the linear model, in the proof of Theorem~\ref{theorem:reduction}, we show that $(1 + \exp(y_i g(a_i;\theta_t)))^{-1} \to \mathbf{1}[y_i g(a_i;\theta_t) < 0]$ when $\gamma \to \infty.$ For the nonlinear models, it is not clear if it is true. Nevertheless, if we assume that $(1 + \exp(y_i g(a_i;\theta_t)))^{-1} \approx \mathbf{1}[y_i g(a_i;\theta_t) < 0],$ then
\begin{align*}
    &\theta_{t+1} \approx \theta_t + \frac{1}{n} \sum_{i \in S_t} \nabla_{\theta} g(a_i;\theta_t), \\ 
    &S_t \eqdef \{i \in [n]\,:y_i \,g(a_i;\theta_t) \leq 0\,\}
\end{align*}
which can be seen as a \emph{generalized perceptron algorithm}. As far as we know, the analysis and convergence rates for this method have not been explored well. 
We believe these questions are important research endeavors.

\textbf{Implicit bias.} 
One of the main features of \refalgone{eq:gd} is \emph{implicit bias}. For small step sizes, $\gamma < \nicefrac{2}{L},$ \citet{soudry2018implicit,ji2018risk} showed that the iterates of \refalgone{eq:gd} not only solve \eqref{eq:linear_task}, but also have a stronger property: $\theta_t \to \theta_*$ when $t \to \infty,$ where $\theta_* = \arg\max_{\norm{\theta} = 1} \min_{i \in [n]} y_i a_i^\top \theta$ is the max-margin/SVM solution. From our observation, for $\gamma \to \infty,$ \refalgone{eq:gd} reduces to \refalgone{eq:batch_perceptron}, which generally does not return the max-margin solution. Therefore, to ensure the implicit bias property, one has to choose $\gamma < \infty,$ and intuitively, the larger $\gamma,$ the slower the convergence to the max-margin solution, but faster convergence to a solution of \eqref{eq:linear_task} according to the experiments and Theorem~\ref{theorem:nonasymfix}.

\textbf{Theoretical guarantees of optimization methods.} 
Typically, when researchers develop new methods, they compare them with previous methods using convergence rates by (loss) function values or by the norm of gradients. 
We believe that this work raises an important concern about this methodology in the context of machine learning tasks. While this work only analyzes the logistic loss with a linear model, the problem can be even more dramatic with more complex losses and models.

\bibliography{aaai25}

\appendix

\onecolumn

\section{Frequently Used Notation}

\newcommand{\ditto}[1][.4pt]{\xrfill{#1}~''~\xrfill{#1}}
\begin{table}[h]
\centering
\begin{tabular}{cc}
\hline
$g = \operatorname{O}(f)$ & Exist $C > 0$ such that $g(z) \leq C \times f(z)$ for all $z \in \mathcal{Z}$\\
$g = \Omega(f)$ & Exist $C > 0$ such that $g(z) \geq C \times f(z)$ for all $z \in \mathcal{Z}$\\
$g = \Theta(f)$ & $g = \operatorname{O}(f)$ and $g = \Omega(f)$ \\
$g = \widetilde{\cO}(f)$ & Exist $C > 0$ such that $g(z) \leq C \times f(z) \times \log (\textnormal{poly}(z))$ for all $z \in \mathcal{Z}$ \\
$g = \widetilde{\Omega}(f)$ & Exist $C > 0$ such that $g(z) \geq C \times f(z) \times \log (\textnormal{poly}(z))$ for all $z \in \mathcal{Z}$ \\
$g = \widetilde{\Theta}(f)$ & $g = \widetilde{\cO}(f)$ and $g = \widetilde{\Omega}(f)$ \\
$\{a, \dots, b\}$ & Set $\{i \in \mathbb{Z}\,|\, a \leq i \leq b\}$ \\
$[n]$ & $\{1, \dots, n\}$ \\
$\N_0$ & $\{0, 1, 2, 3, \dots\}$ \\
\hline
\end{tabular}
\end{table}

\section{Proof of Theorem~\ref{theorem:perceptron}}
\label{sec:theorem:perceptron}

\BATCHPERCEPTRON*

\begin{proof}
    We use a slight modification of the classical arguments \citep{novikoff1962convergence,pattern}. Let us take $\hat{\theta}_* = \arg\max_{\norm{\theta} = 1} \min_{i \in [n]} y_i a_i^\top \theta$ and $\alpha > 0.$ Then, using simple algebra, for the first $t \geq 1$ such that $\abs{S_t} \neq 0,$ we obtain
    \begin{align*}
        &\norm{\hat{\theta}_{t+1} - \alpha \hat{\theta}_*}^2 = \norm{\hat{\theta}_{t} + \frac{1}{n} \sum_{i \in S_t}  y_i a_i - \alpha \hat{\theta}_*}^2 = \norm{\hat{\theta}_{t} - \alpha \hat{\theta}_*}^2 + 2 \inp{\hat{\theta}_{t} - \alpha \hat{\theta}_*}{\frac{1}{n} \sum_{i \in S_t}  y_i a_i} + \norm{\frac{1}{n} \sum_{i \in S_t}  y_i a_i}^2.
    \end{align*}
    By the definition of $S_t,$ we have $\hat{\theta}_{t}^\top y_i a_i \leq 0$ for all $i \in S_t.$ Moreover, by the definition of $\hat{\theta}_*,$ we get
    $\hat{\theta}_{*}^\top y_i a_i \geq \mu$ for all $i \in [n].$ Therefore
    \begin{align*}
        \norm{\hat{\theta}_{t+1} - \alpha \hat{\theta}_*}^2 
        &\leq \norm{\hat{\theta}_{t} - \alpha \hat{\theta}_*}^2 - \frac{2 \alpha \abs{S_t} \mu}{n} + \norm{\frac{1}{n} \sum_{i \in S_t}  y_i a_i}^2 \\
        &\leq \norm{\hat{\theta}_{t} - \alpha \hat{\theta}_*}^2 - \frac{2 \alpha \abs{S_t} \mu}{n} + \frac{\abs{S_t}^2}{n^2} \max_{i \in [n]}\norm{ a_i}^2,
    \end{align*}
    where we use Jensen's inequality. 
    Taking $\alpha = \frac{\max_{i \in [n]}\norm{ a_i}^2}{\mu},$ we get
    \begin{align*}
        \norm{\hat{\theta}_{t+1} - \alpha \hat{\theta}_*}^2 
        &\leq \norm{\hat{\theta}_{t} - \alpha \hat{\theta}_*}^2 - \frac{2 \abs{S_t}}{n} \max\limits_{i \in [n]}\norm{ a_i}^2 + \frac{\abs{S_t}^2}{n^2} \max_{i \in [n]}\norm{ a_i}^2.
    \end{align*}
    Note that $-\frac{2 x}{n} + \frac{x^2}{n^2} \leq -\frac{2}{n} + \frac{1}{n^2}$ for all $x \in \{1, \dots, n\}.$ Thus
    \begin{align}
        \norm{\hat{\theta}_{t+1} - \alpha \hat{\theta}_*}^2 
        &\leq \norm{\hat{\theta}_{t} - \alpha \hat{\theta}_*}^2 - \frac{2}{n} \max\limits_{i \in [n]}\norm{ a_i}^2 + \frac{1}{n^2} \max_{i \in [n]}\norm{ a_i}^2 \nonumber \\
        &\leq \norm{\hat{\theta}_{t} - \alpha \hat{\theta}_*}^2 - \frac{1}{n} \max\limits_{i \in [n]}\norm{ a_i}^2 \leq \norm{\hat{\theta}_{1} - \alpha \hat{\theta}_*}^2 - \frac{t}{n} \max\limits_{i \in [n]}\norm{ a_i}^2.
        \label{eq:LSHXXavAiWQYDxiJ}
    \end{align}
    Using the same reasoning, for $t = 0,$ we have 
    \begin{align*}
        \norm{\hat{\theta}_{1} - \alpha \hat{\theta}_*}^2 &= \norm{\hat{\theta}_{0} - \alpha \hat{\theta}_*}^2 + 2 \inp{\hat{\theta}_{0} - \alpha \hat{\theta}_*}{\frac{1}{2 n} \sum_{i=1}^n  y_i a_i} + \norm{\frac{1}{2 n} \sum_{i=1}^n y_i a_i}^2 \\
        &\leq \norm{\hat{\theta}_{0} - \alpha \hat{\theta}_*}^2 - \alpha \mu + \frac{1}{4} \max\limits_{i \in [n]}\norm{ a_i}^2 < \norm{\hat{\theta}_{0} - \alpha \hat{\theta}_*}^2.
    \end{align*}
    We substitute the last inequality to \eqref{eq:LSHXXavAiWQYDxiJ} and get
    \begin{align*}
        \norm{\hat{\theta}_{t+1} - \alpha \hat{\theta}_*}^2 
        &< \norm{\hat{\theta}_{0} - \alpha \hat{\theta}_*}^2 - \frac{t}{n} \max\limits_{i \in [n]}\norm{ a_i}^2.
    \end{align*}
    Therefore, the algorithm necessarily solves \eqref{eq:linear_task} after at most 
    \begin{align*}
    \frac{n \norm{\hat{\theta}_{0} - \alpha \hat{\theta}_*}^2}{\max\limits_{i \in [n]}\norm{ a_i}^2}
    \end{align*} iterations. 
    It is left substitute the choice of $\alpha$ and $\hat{\theta}_{0},$ and use the fact $\norm{\hat{\theta}_*} = 1.$
\end{proof}

\section{Proof of Theorem~\ref{theorem:nonasymfix}}

\BATCHPERCEPTRONNONASYMPTFIX*

The proof employs a modification of the arguments presented in \citep{novikoff1962convergence,pattern}.

\begin{proof}
    Let us take $\hat{\theta}_* = \arg\max_{\norm{\theta} = 1} \min_{i \in [n]} y_i a_i^\top \theta$ and $\alpha > 0.$ Assume that $t^*$ is the first moment when the algorithm find $\theta_t$ such that $y_i a_i^\top \theta_t > 0$ for all $i \in [n].$ If $t^* \leq \eqref{eq:upper_bound},$ then the theorem is proved. Let $t^* > \eqref{eq:upper_bound},$ then we will show a contradiction.
    
    For all $0 \leq t < t^*$, using simple algebra, we obtain
    \begin{align}
        \label{eq:qogboVymsaWjOgbn}
        &\norm{\frac{\theta_{t+1}}{\gamma} - \alpha \hat{\theta}_*}^2 
        = \norm{\frac{\theta_{t}}{\gamma} - \alpha \hat{\theta}_*}^2 + 2 \inp{\frac{\theta_{t}}{\gamma} - \alpha \hat{\theta}_*}{-\beta_t \nabla f(\theta_t)} + \norm{\beta_t \nabla f(\theta_t)}^2.
    \end{align}
    Let us temporary define $w_i \eqdef (1 + \exp(y_i a_i^\top \theta_t))^{-1}.$ Then $\beta_t = \left(\frac{1}{n} \sum_{i=1}^n w_i\right)^{-1},$
    \begin{align*}
        \nabla f(\theta_t) = - \frac{1}{n} \sum_{i=1}^n w_i y_i a_i
    \end{align*}
    and 
    \begin{align}
        \label{eq:odMAtBBIoaTVxDQMCpO}
        \norm{\beta_t \nabla f(\theta_t)}^2 = \norm{\left(\frac{1}{n} \sum_{i=1}^n w_i\right)^{-1} \frac{1}{n} \sum_{i=1}^n w_i y_i a_i}^2 \leq R^2
    \end{align}
    due to Jensen's inequality. 
    By the definition of $\hat{\theta}_*,$ $y_i a_i^\top \hat{\theta}_{*} \geq \mu$ for all $i \in [n].$ Thus
    \begin{align}
        \label{eq:UjiKdYuZIKsiMoPZlPL}
        -\beta_t \hat{\theta}_*^\top \nabla f(\theta_t) = \beta_t \left(\frac{1}{n} \sum_{i=1}^n w_i y_i a_i^\top \hat{\theta}_*\right) \geq \beta_t \left(\frac{1}{n} \sum_{i=1}^n w_i\right) \mu = \mu
    \end{align}
    On the other hand
    \begin{align*}
        -\beta_t \theta_t^\top \nabla f(\theta_t) = \left(\frac{1}{n} \sum_{i=1}^n w_i\right)^{-1} \frac{1}{n} \sum_{i=1}^n w_i \log\left(\frac{1}{w_i} - 1\right) \leq \log\left(\frac{n}{\sum_{i=1}^n w_i} - 1\right)
    \end{align*}
    since $\log\left(\frac{1}{w_i} - 1\right) = y_i a_i^\top \theta_t,$ and due to Jensen's inequality. Since $t < t^*,$ the problem is not solved at the iteration $t,$ there exists $i \in [n]$ such that $y_i a_i^\top \theta_t\leq 0.$ Thus $\sum_{i=1}^n w_i \geq 1 / 2$ and 
    \begin{align}
        \label{eq:jdzVkkdm}
        -\beta_t \theta_t^\top \nabla f(\theta_t) \leq \log(2 n - 1).
    \end{align}
    We substitute \eqref{eq:odMAtBBIoaTVxDQMCpO}, \eqref{eq:UjiKdYuZIKsiMoPZlPL}, and \eqref{eq:jdzVkkdm} to \eqref{eq:qogboVymsaWjOgbn}, and get
    \begin{align*}
        &\norm{\frac{\theta_{t+1}}{\gamma} - \alpha \hat{\theta}_*}^2 
        \leq \norm{\frac{\theta_{t}}{\gamma} - \alpha \hat{\theta}_*}^2 - 2 \alpha \mu + \frac{2 \log(2 n - 1)}{\gamma} + R^2.
    \end{align*}
    Taking $\alpha = \frac{\frac{2 \log(2 n - 1)}{\gamma} + R^2}{\mu},$ we obtain
    \begin{align*}
        &\norm{\frac{\theta_{t+1}}{\gamma} - \alpha \hat{\theta}_*}^2 
        \leq \norm{\frac{\theta_{t}}{\gamma} - \alpha \hat{\theta}_*}^2 - \alpha \mu \leq \norm{\frac{\theta_{0}}{\gamma} - \alpha \hat{\theta}_*}^2 - (t + 1) \alpha \mu.
    \end{align*}
    For all $0 \leq t < t^*,$ the last inequality necessarily yields 
    \begin{align*}
        t + 1 \leq \frac{\norm{\frac{\theta_{0}}{\gamma} - \alpha \hat{\theta}_*}^2}{\alpha \mu} = \frac{\alpha}{\mu} = \frac{R^2}{\mu^2} + \frac{2\log(2n - 1)}{\gamma \mu^2},
    \end{align*}
    where we use the choice of $\alpha, \hat{\theta}_*,$ and $\theta_0.$ Taking $t = t^* - 1,$ we get the inequality
    \begin{align*}
        t^* \leq \frac{R^2}{\mu^2} + \frac{2\log(2n - 1)}{\gamma \mu^2}
    \end{align*}
    that contradicts the assumption at the beginning of the proof.
\end{proof}

\section{Proof of Theorem~\ref{theorem:normalized_perceptron}}
\label{sec:normalized_perceptron}

\NORMALIZEDPERCERTRON*

\begin{proof}
    Let us take $\hat{\theta}_* = \arg\max_{\norm{\theta} = 1} \min_{i \in [n]} y_i a_i^\top \theta$ and $\alpha > 0.$ Using the same idea as in the proof of Theorem~\ref{theorem:perceptron}, we get
    \begin{align*}
        &\norm{\hat{\theta}_{t+1} - \alpha \hat{\theta}_*}^2 = \norm{\hat{\theta}_{t} - \alpha \hat{\theta}_*}^2 + 2 \inp{\hat{\theta}_{t} - \alpha \hat{\theta}_*}{\frac{1}{\abs{S_t}} \sum_{i \in S_t}  y_i a_i} + \norm{\frac{1}{\abs{S_t}} \sum_{i \in S_t}  y_i a_i}^2
    \end{align*}
    for the first $t \geq 0$ such that $\abs{S_t} \neq 0.$
    By the definition of $S_t,$ we have $y_i a_i^\top \hat{\theta}_{t} \leq 0$ for all $i \in S_t.$ Moreover, by the definition of $\hat{\theta}_*,$ we get
    $y_i a_i^\top \hat{\theta}_{*} \geq \mu$ for all $i \in [n].$ Therefore
    \begin{align*}
        \norm{\hat{\theta}_{t+1} - \alpha \hat{\theta}_*}^2 
        &\leq \norm{\hat{\theta}_{t} - \alpha \hat{\theta}_*}^2 - 2 \alpha \mu + \norm{\frac{1}{\abs{S_t}} \sum_{i \in S_t}  y_i a_i}^2 \\
        &\leq \norm{\hat{\theta}_{t} - \alpha \hat{\theta}_*}^2 - 2 \alpha \mu + \max_{i \in [n]}\norm{ a_i}^2,
    \end{align*}
    where we use Jensen's inequality. 
    Taking $\alpha = \max_{i \in [n]}\norm{ a_i}^2 / \mu,$ we get
    \begin{align*}
        \norm{\hat{\theta}_{t+1} - \alpha \hat{\theta}_*}^2 
        &\leq \norm{\hat{\theta}_{t} - \alpha \hat{\theta}_*}^2 - \max_{i \in [n]}\norm{ a_i}^2 \leq \norm{\hat{\theta}_{0} - \alpha \hat{\theta}_*}^2 - (t + 1) \max_{i \in [n]}\norm{ a_i}^2.
    \end{align*}
    Therefore, the algorithm solves \eqref{eq:linear_task} after at most $\norm{\hat{\theta}_{0} - \alpha \hat{\theta}_*}^2 / \max_{i \in [n]}\norm{a_i}^2 = \max_{i \in [n]}\norm{a_i}^2 / \mu^2$ iterations.
\end{proof}

\section{Extra Experiments and Details}
\label{sec:exp_setup}
The code was written in Python. A distributed environment was emulated on a machine with Intel(R) Xeon(R) Gold 6226R CPU @ 2.90GHz and 64 cores. 

By default, we consider the following experimental setup in the paper. In \emph{CIFAR-10} and \emph{EuroSAT}, we take the first two classes. In \emph{FashionMNIST}, we take class $0$ and class $4.$ In \emph{MNIST}, we choose class $7$ and class $8.$ In the latter two datasets, we tried to choose visually more challenging classes to classify. We subsample $5\,000$ data points in each dataset to ensure we work with separable data.

We also provide additional experiments with $1\,000$ and $10\,000$ data points in Section~\ref{sec:more_data_points} to support our experimental results from the main part.

\subsection{Construction of imbalanced datasets}
\label{sec:exp_setup_imbalanced}
To demonstrate the superiority and robustness of \refalgone{eq:norm_perceptron} and \refalgone{eq:norm_gd} to imbalanced datasets in Figure~\ref{fig:fashion_mnist_compare_perceptron_imbalanced} and Table~\ref{eq:tbl}, we construct imbalanced versions of the datasets in the following way. The setup here is almost the same as in the default setup: we take the same two classes and subsample 10\% of data points, but then we repeat each sample of one class $10$ times. The last step is the main difference that helps to construct imbalanced datasets.

\subsection{Experiments on other datasets to support Section~\ref{sec:gd_is_perceptron}}

We now provide additional experiments on other datasets to support our results from Section~\ref{sec:gd_is_perceptron}:

\begin{figure}[H]
\centering
\begin{subfigure}[t]{0.4\textwidth}
    \centering
    \includegraphics[width=\textwidth]{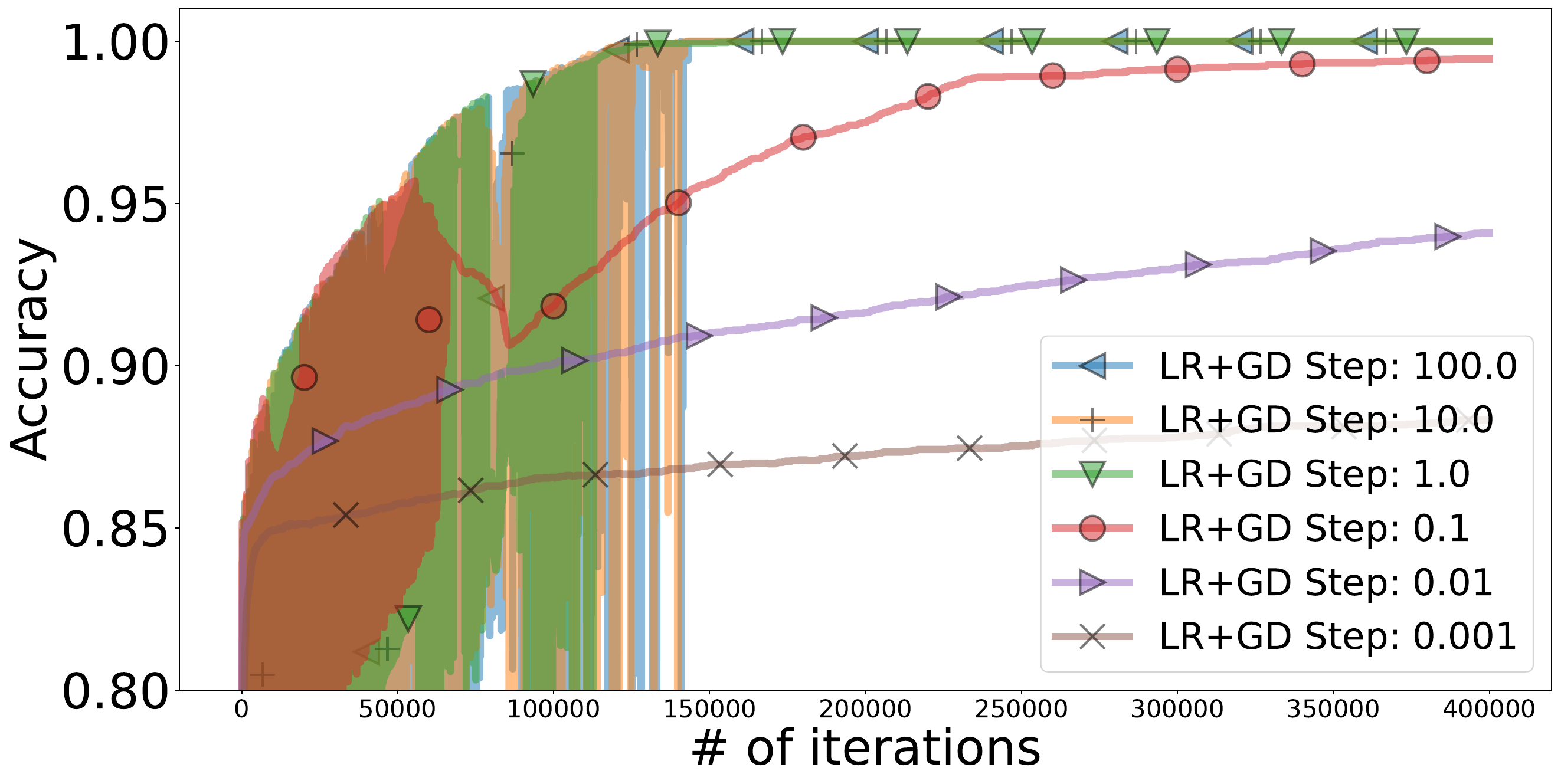}
\end{subfigure}
\begin{subfigure}[t]{0.4\textwidth}
    \centering
    \includegraphics[width=\textwidth]{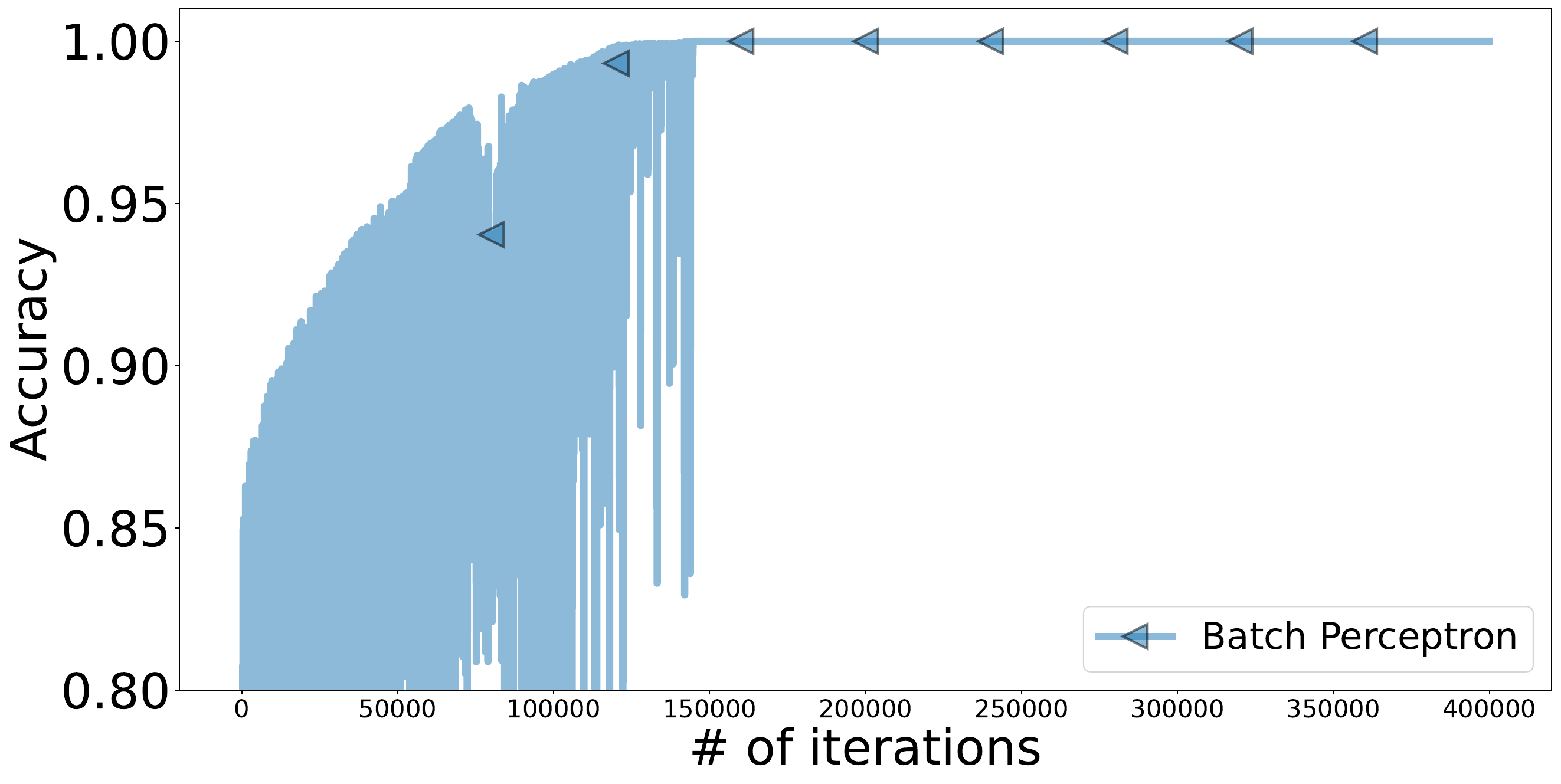}
\end{subfigure}
\caption{We show that \protect\refalgone{eq:gd} with large step sizes aligns with \protect\refalgone{eq:batch_perceptron} on \emph{EuroSAT}.}
\label{fig:eurosat}
\centering
\begin{subfigure}[t]{0.32\textwidth}
    \centering
    \includegraphics[width=\textwidth]{./results_2024/gd_two_layer_eos_linear_one_no_bias_loss_bce_logits_eurosat_num_samples_5000_more_iters_more_iters_filter_classes_0_1_accuracy.pdf}
\end{subfigure}
\begin{subfigure}[t]{0.32\textwidth}
    \centering
    \includegraphics[width=\textwidth]{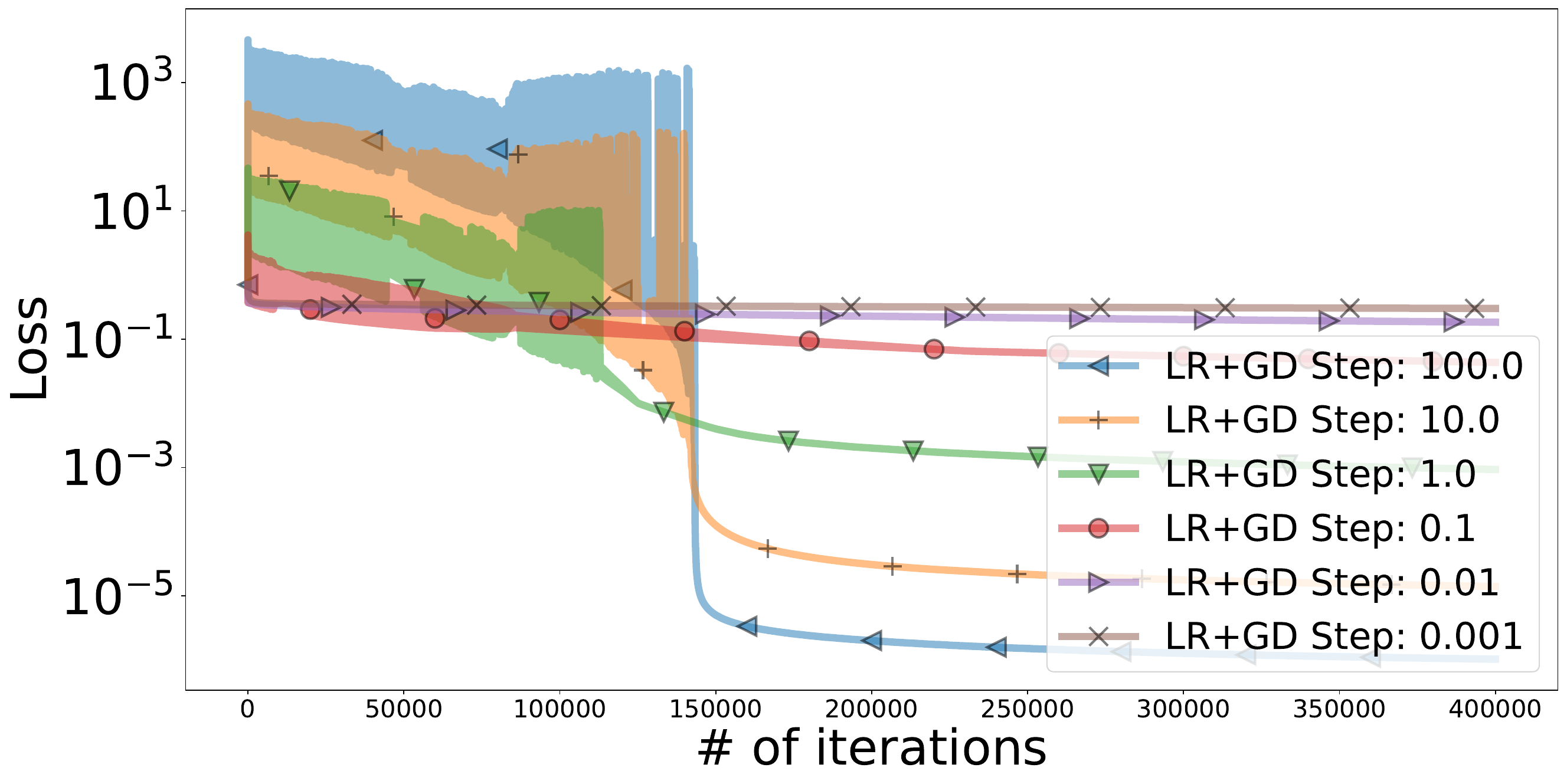}
\end{subfigure}
\begin{subfigure}[t]{0.32\textwidth}
    \centering
    \includegraphics[width=\textwidth]{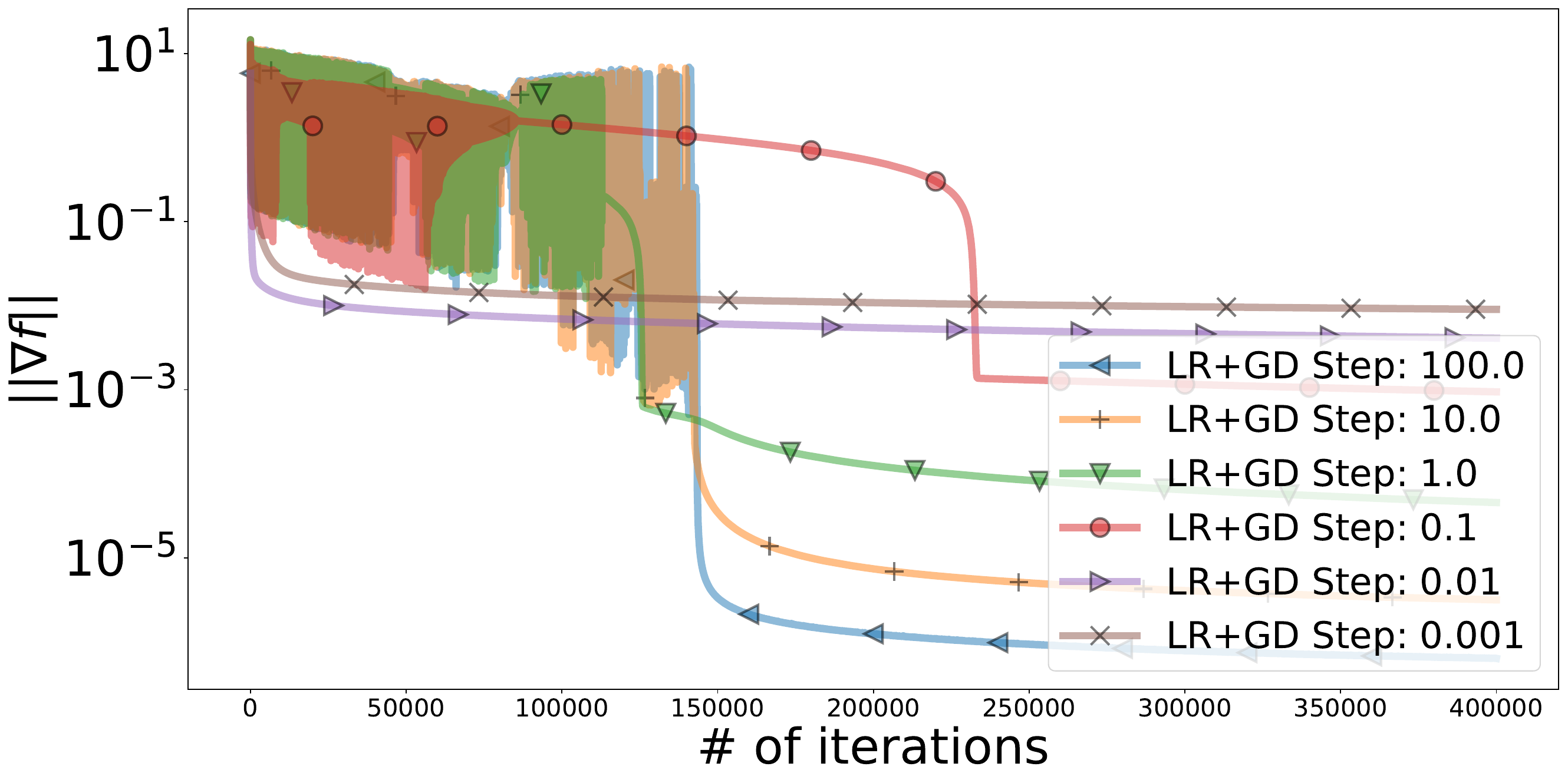}
\end{subfigure}
\caption{Accuracy, function values, and the norm of gradients of the logistic loss \eqref{eq:logistic_regression} on \emph{EuroSAT} during the runs of \protect\refalgone{eq:gd}.}
\label{fig:eurosat_more}
\centering
\begin{subfigure}[t]{0.4\textwidth}
    \centering
    \includegraphics[width=\textwidth]{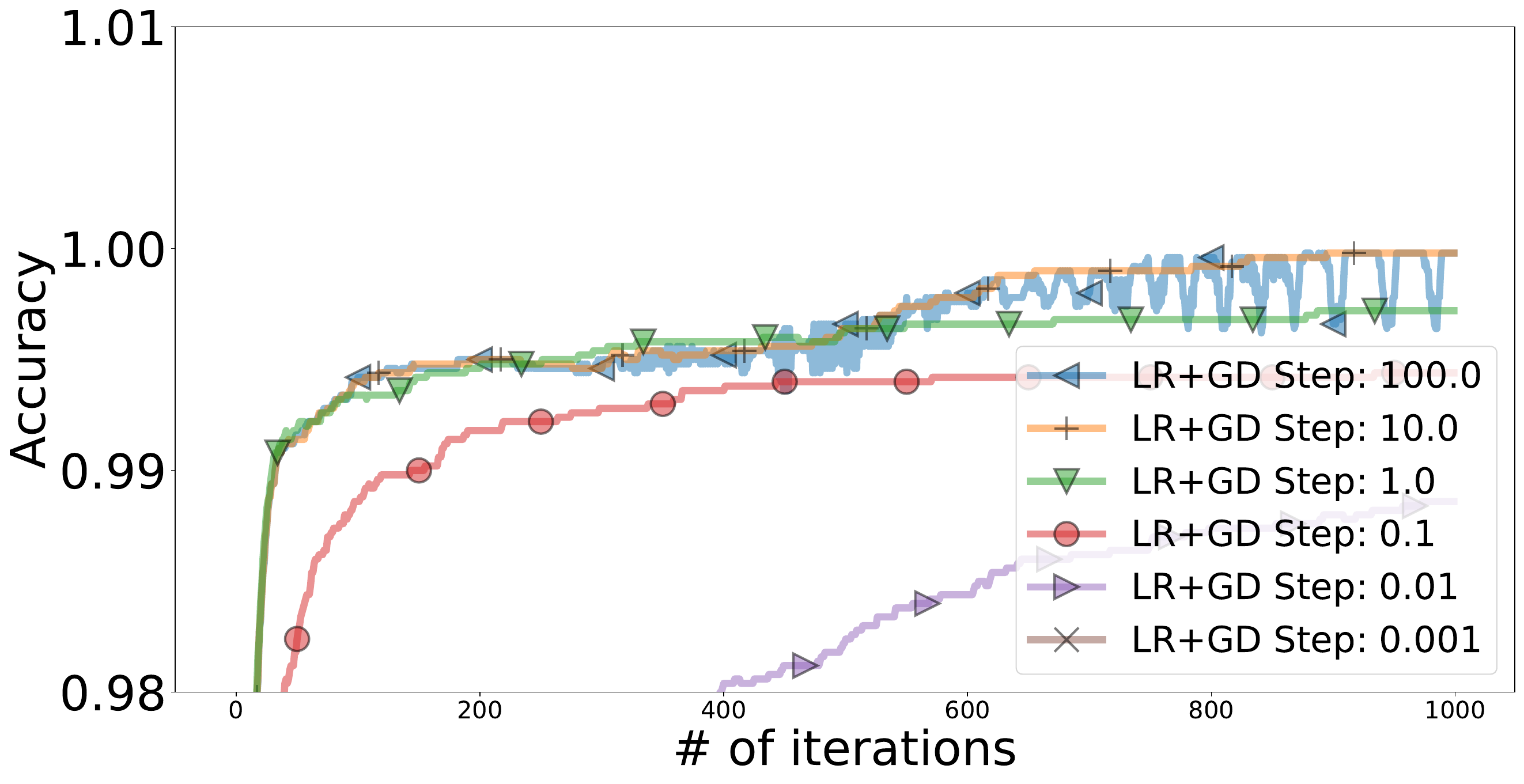}
\end{subfigure}
\begin{subfigure}[t]{0.4\textwidth}
    \centering
    \includegraphics[width=\textwidth]{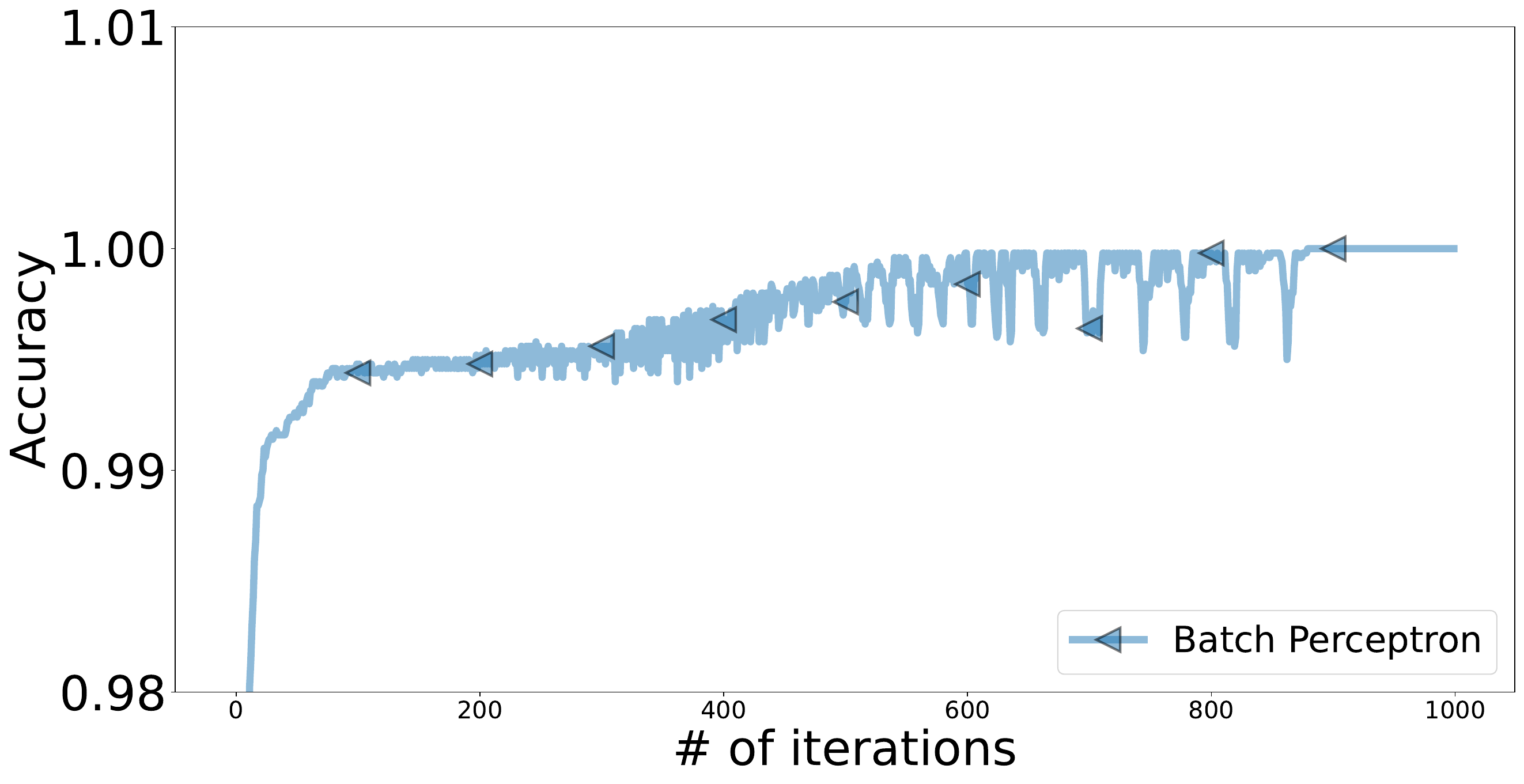}
\end{subfigure}
\caption{We show that \protect\refalgone{eq:gd} with large step sizes aligns with \protect\refalgone{eq:batch_perceptron} on \emph{MNIST}.}
\label{fig:mnist}
\begin{subfigure}[t]{0.32\textwidth}
    \centering
    \includegraphics[width=\textwidth]{./results_2024/gd_two_layer_eos_linear_one_no_bias_loss_bce_logits_mnist_num_samples_5000_more_iters_more_iters_filter_classes_7_8_accuracy.pdf}
\end{subfigure}
\begin{subfigure}[t]{0.32\textwidth}
    \centering
    \includegraphics[width=\textwidth]{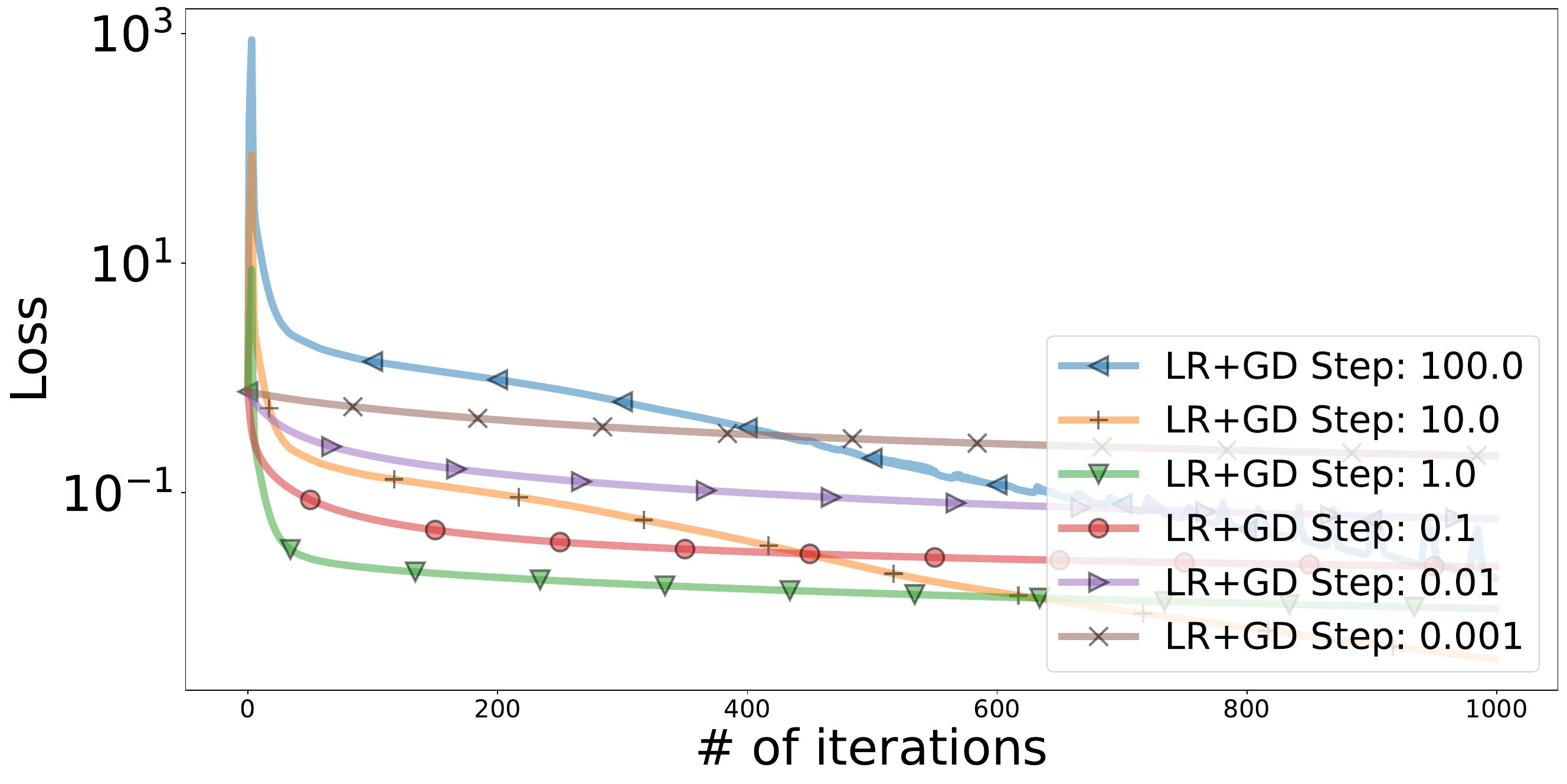}
\end{subfigure}
\begin{subfigure}[t]{0.32\textwidth}
    \centering
    \includegraphics[width=\textwidth]{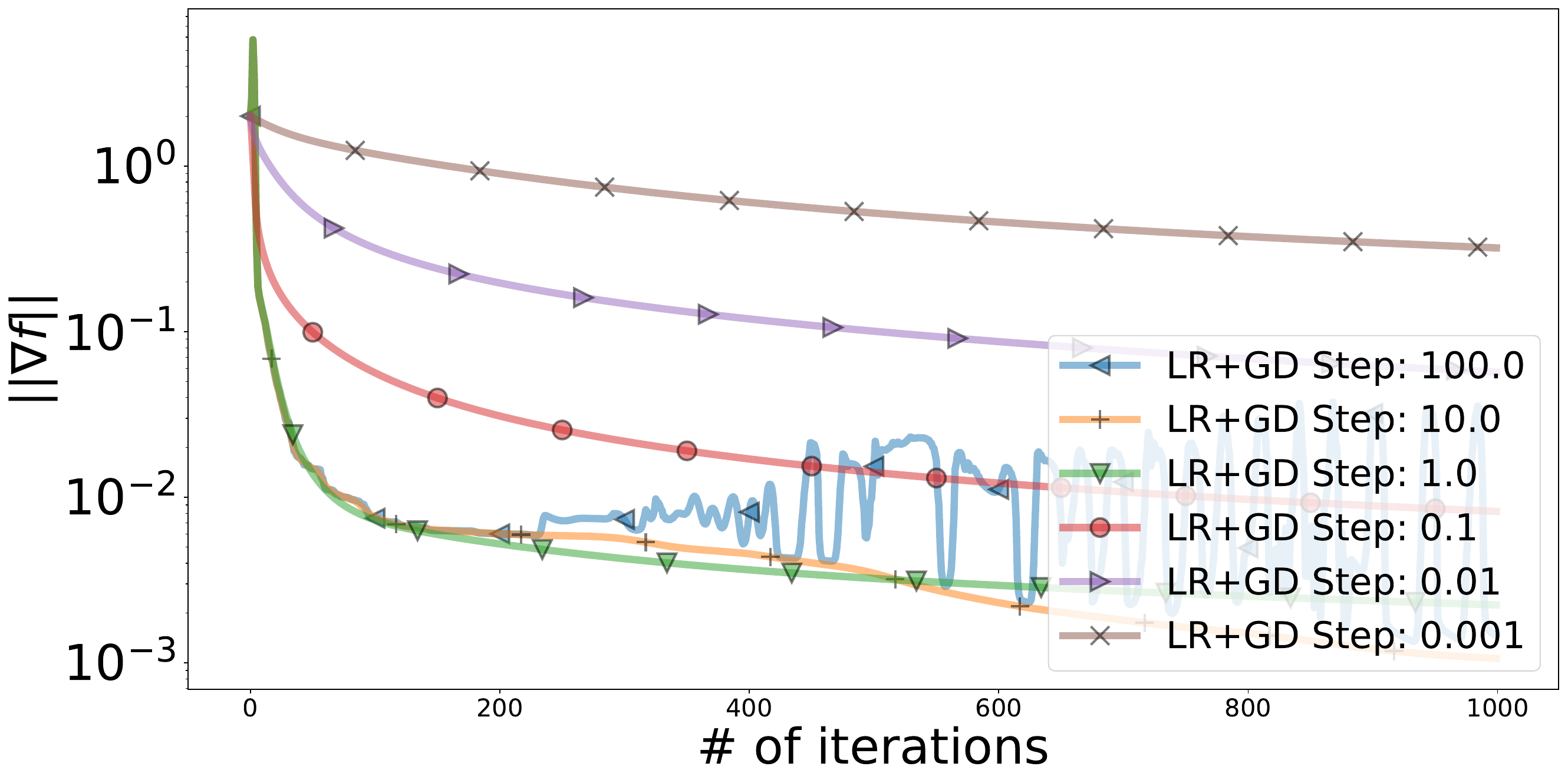}
\end{subfigure}
\caption{Accuracy, function values, and the norm of gradients of the logistic loss \eqref{eq:logistic_regression} on \emph{MNIST} during the runs of \protect\refalgone{eq:gd}.}
\label{fig:mnist_more}
\end{figure}
\subsection{Experiments with various numbers of samples to support Section~\ref{sec:gd_is_perceptron}}
\label{sec:more_data_points}
    \begin{figure}[h]
    \centering
    \begin{subfigure}[t]{0.49\textwidth}
        \centering
        \includegraphics[width=\textwidth]{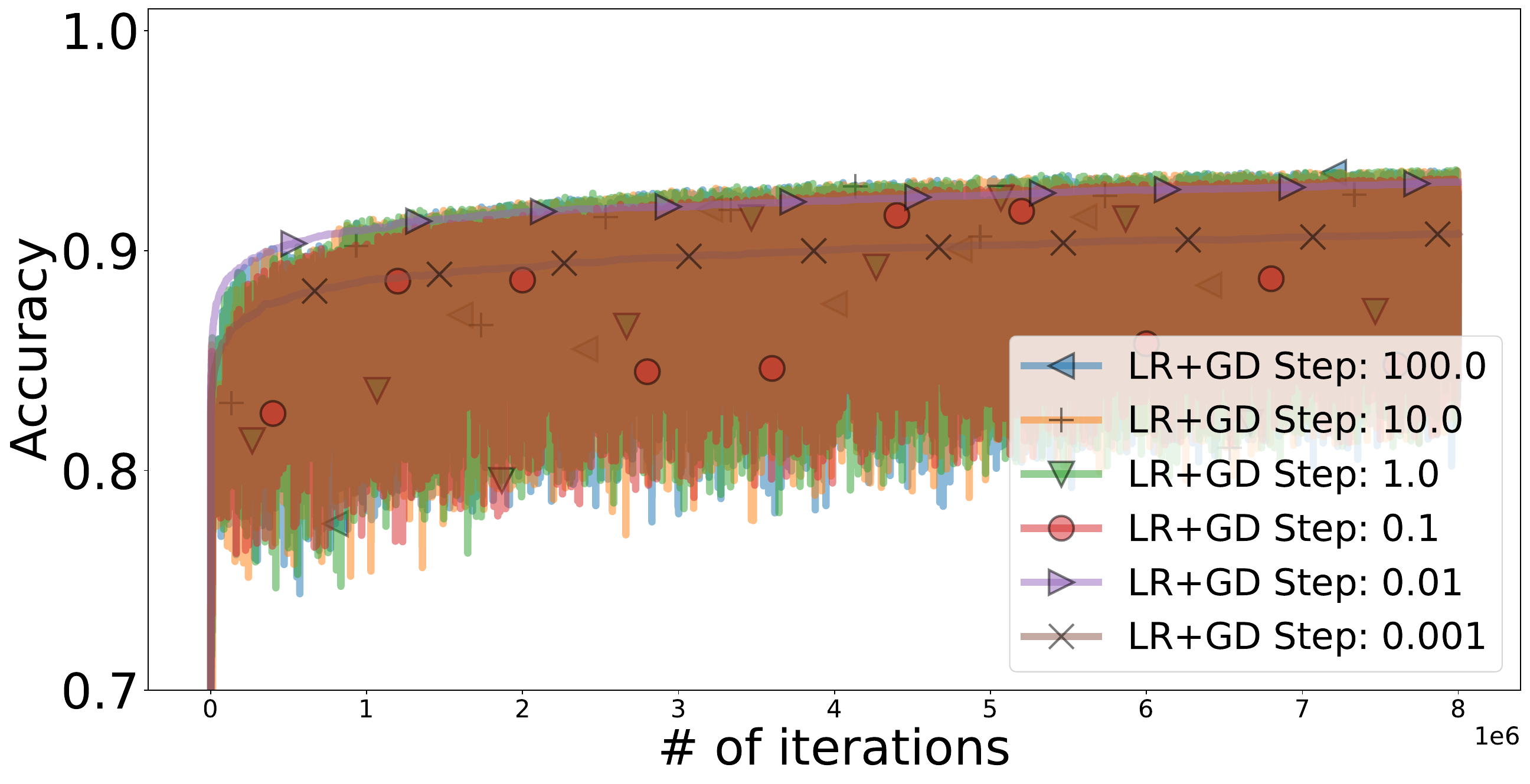}
    \end{subfigure}
    \begin{subfigure}[t]{0.49\textwidth}
        \centering
        \includegraphics[width=\textwidth]{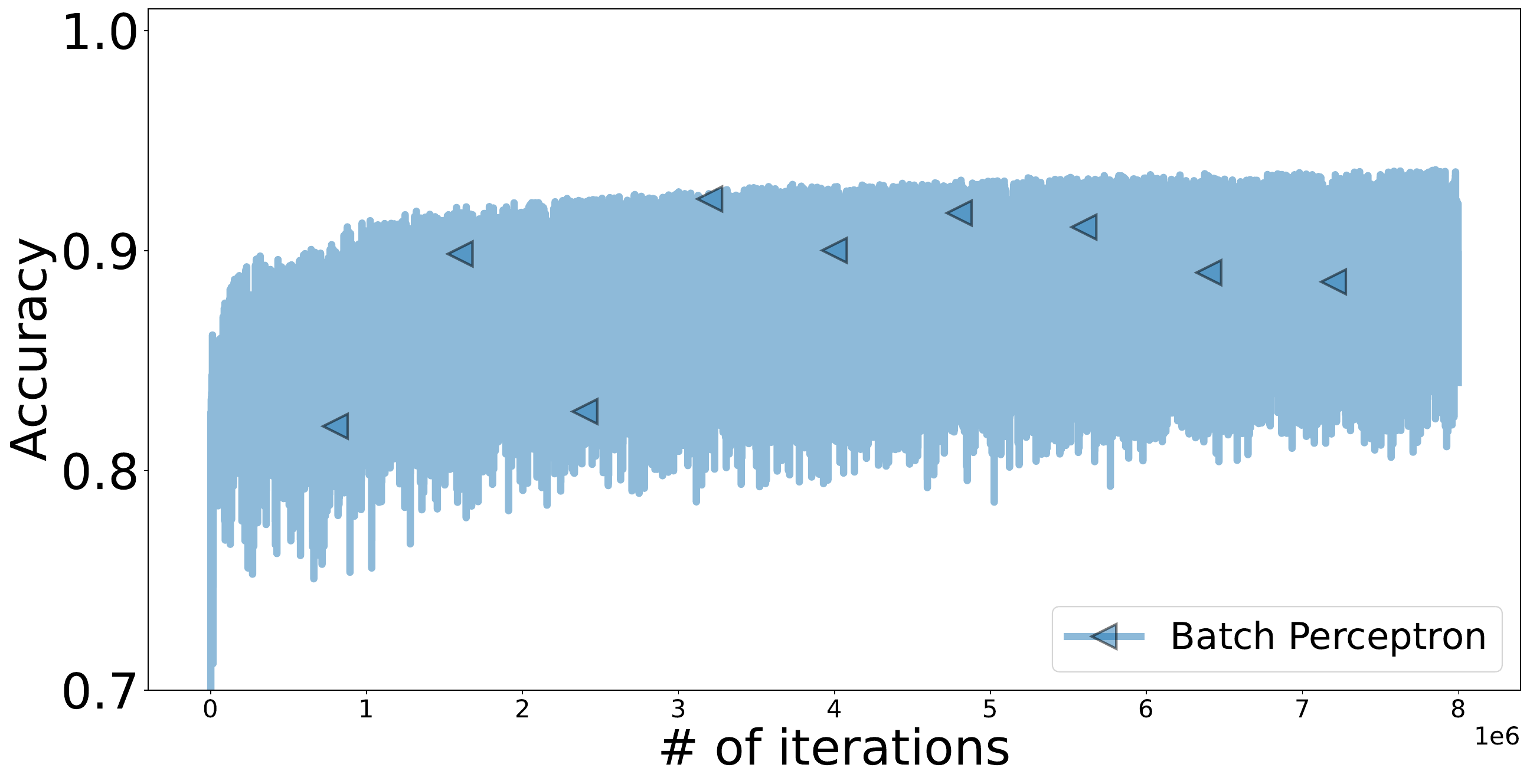}
    \end{subfigure}
    \caption{In these plots we repeat the experiments from Figure~\ref{fig:cifar10} on \emph{CIFAR-10} but with $10\,000$ samples. When \# of samples is $10\,000,$ we did not wait for the moment the algorithm achieves Accuracy=$1.0$ (time out).}
    \centering
    \begin{subfigure}[t]{0.32\textwidth}
        \centering
        \includegraphics[width=\textwidth]{./results_2024/gd_two_layer_eos_linear_one_no_bias_loss_bce_logits_num_samples_10000_more_iters_more_iters_filter_classes_0_1_repeat_longer_accuracy.pdf}
    \end{subfigure}
    \begin{subfigure}[t]{0.32\textwidth}
        \centering
        \includegraphics[width=\textwidth]{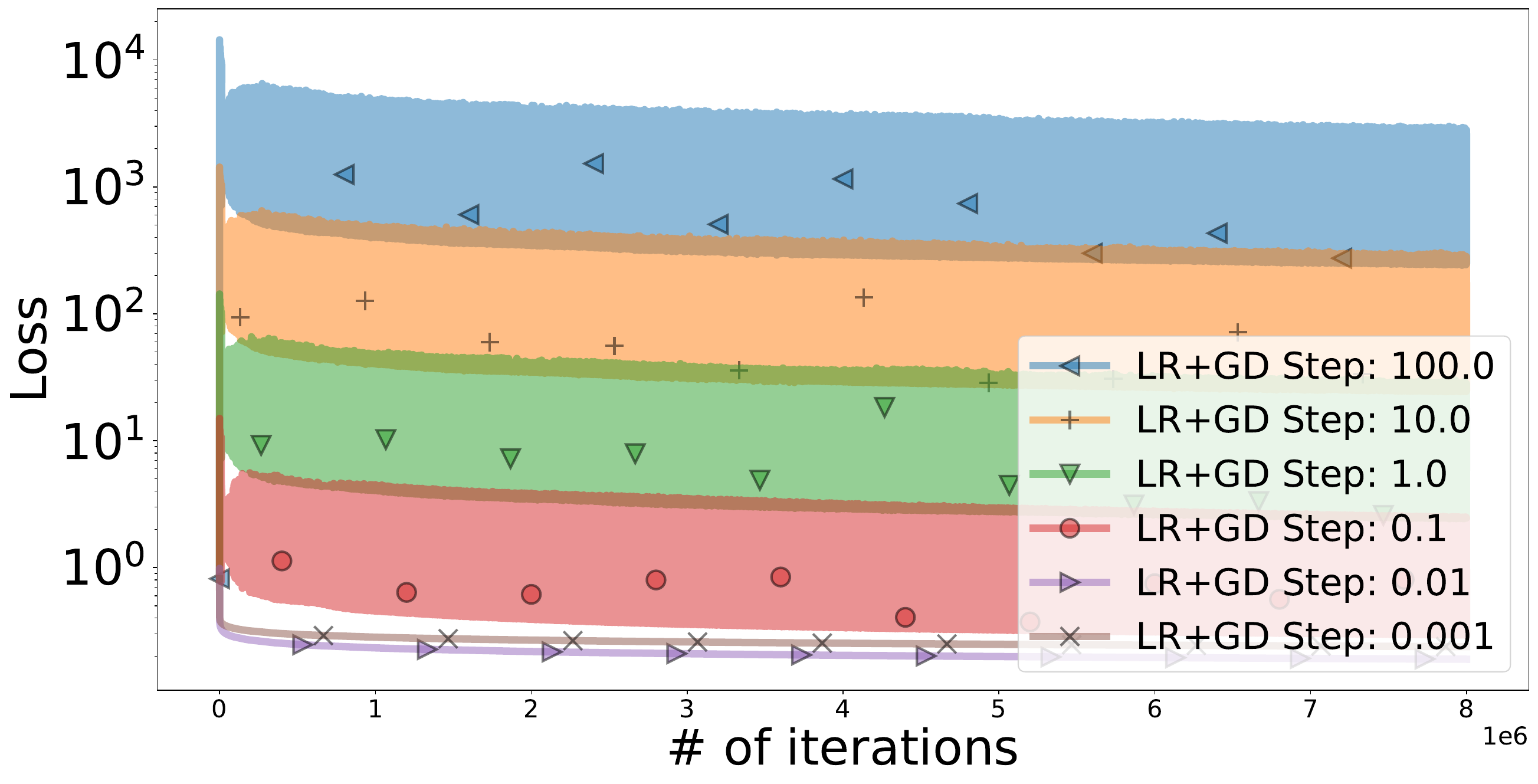}
    \end{subfigure}
    \begin{subfigure}[t]{0.32\textwidth}
        \centering
        \includegraphics[width=\textwidth]{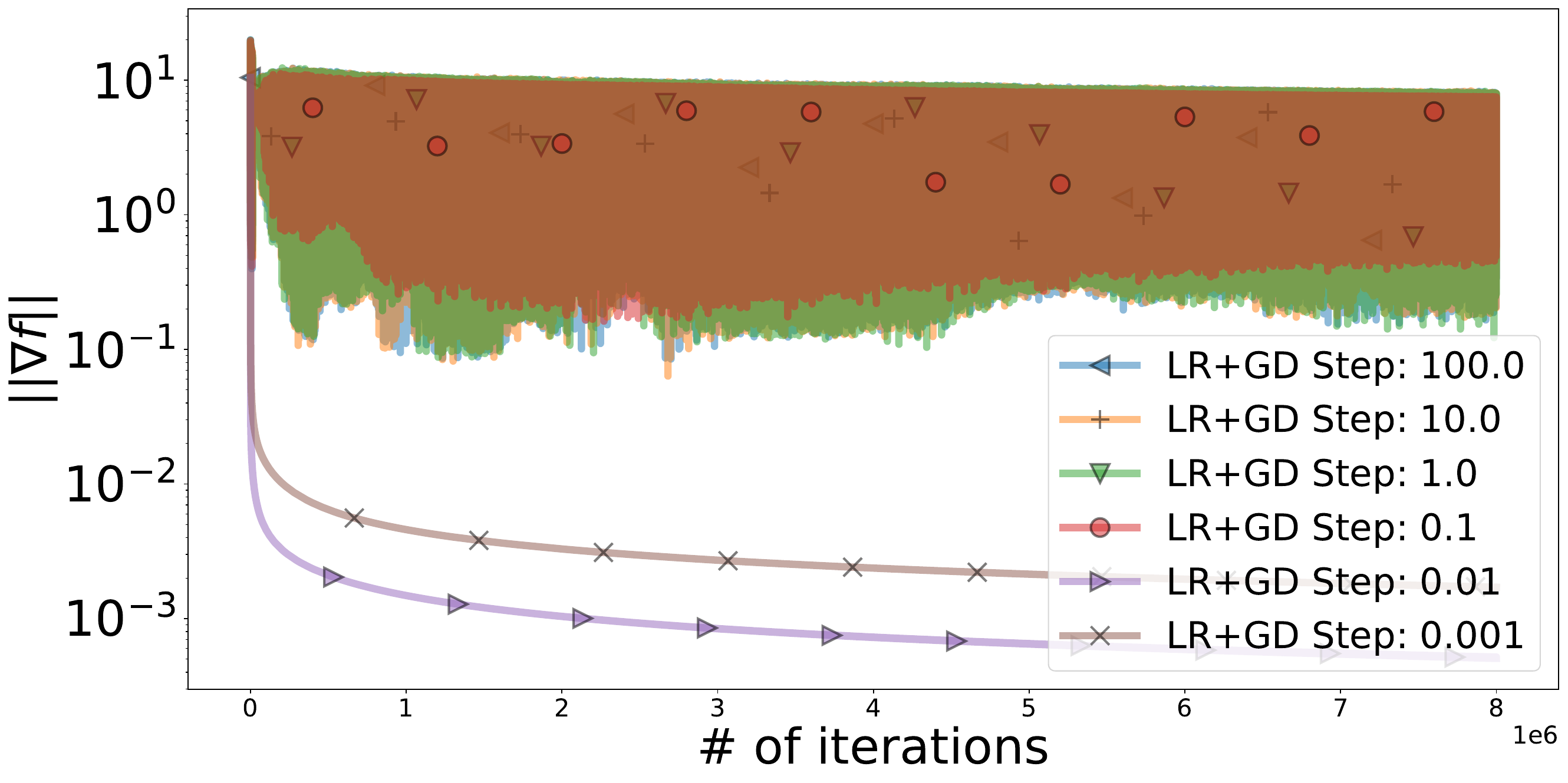}
    \end{subfigure}
    \caption{Accuracy, function values, and the norm of gradients of the logistic loss \eqref{eq:logistic_regression} on \emph{CIFAR-10} but with $10\,000$ samples.}
    \end{figure}
    \begin{figure}[h]
    \centering
    \begin{subfigure}[t]{0.49\textwidth}
        \centering
        \includegraphics[width=\textwidth]{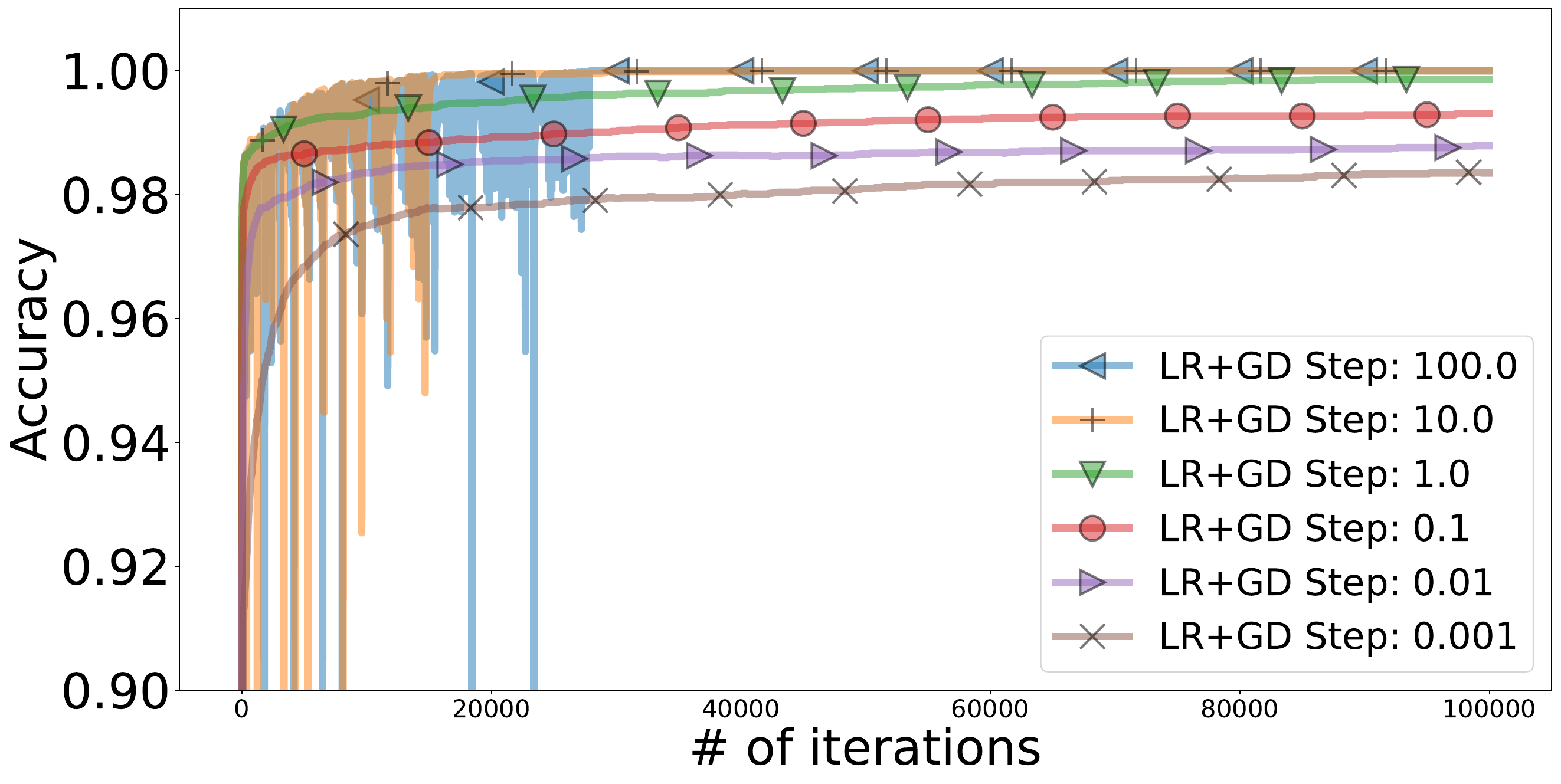}
    \end{subfigure}
    \begin{subfigure}[t]{0.49\textwidth}
        \centering
        \includegraphics[width=\textwidth]{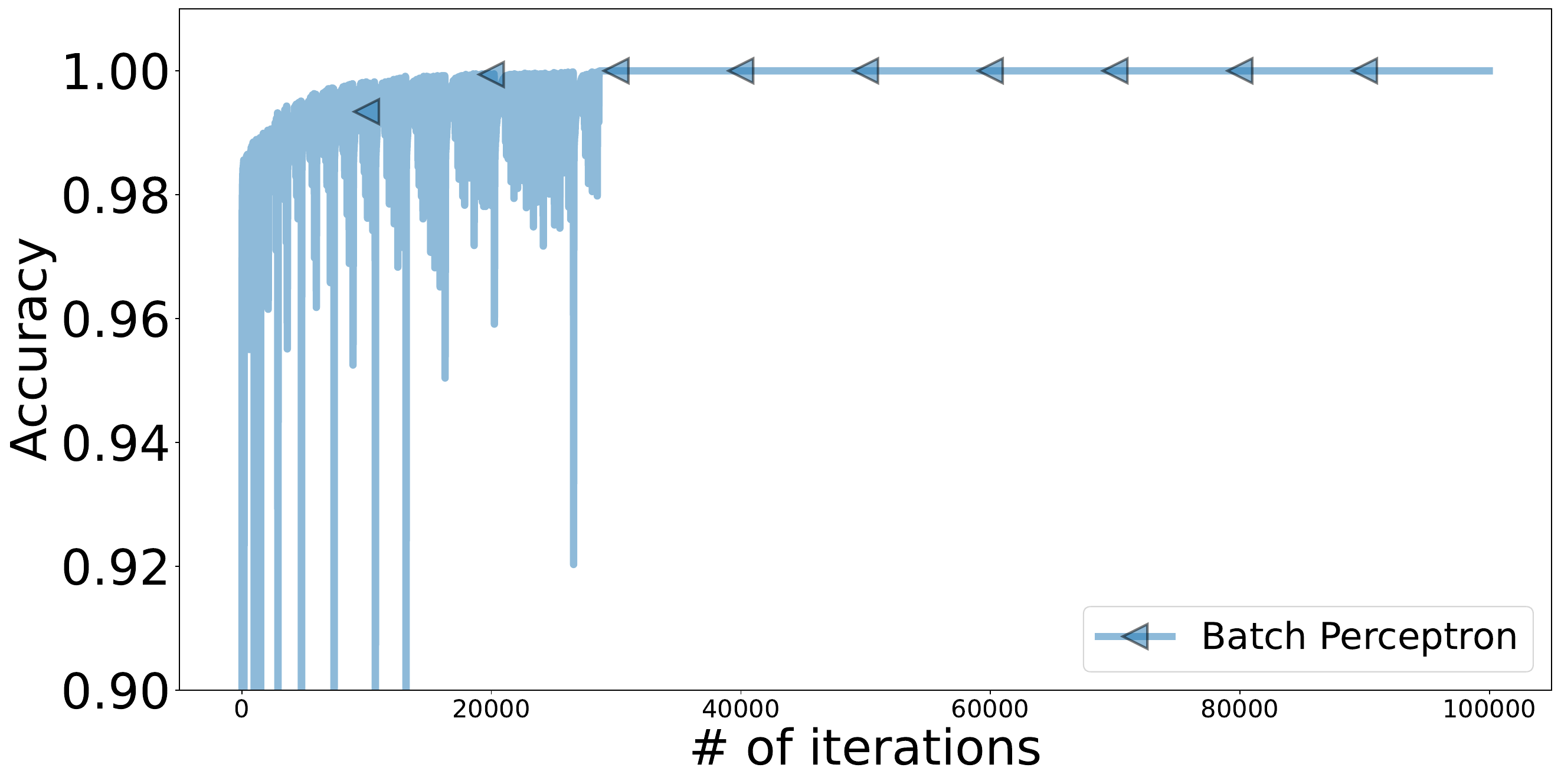}
    \end{subfigure}
    \caption{In these plots we repeat the experiments from Figure~\ref{fig:fashionmnist} on \emph{FashionMNIST} but with $10\,000$ samples.}
    \centering
    \begin{subfigure}[t]{0.32\textwidth}
        \centering
        \includegraphics[width=\textwidth]{./results_2024/gd_two_layer_eos_linear_one_no_bias_loss_bce_logits_fashion_mnist_num_samples_10000_more_iters_more_iters_filter_classes_0_4_accuracy.pdf}
    \end{subfigure}
    \begin{subfigure}[t]{0.32\textwidth}
        \centering
        \includegraphics[width=\textwidth]{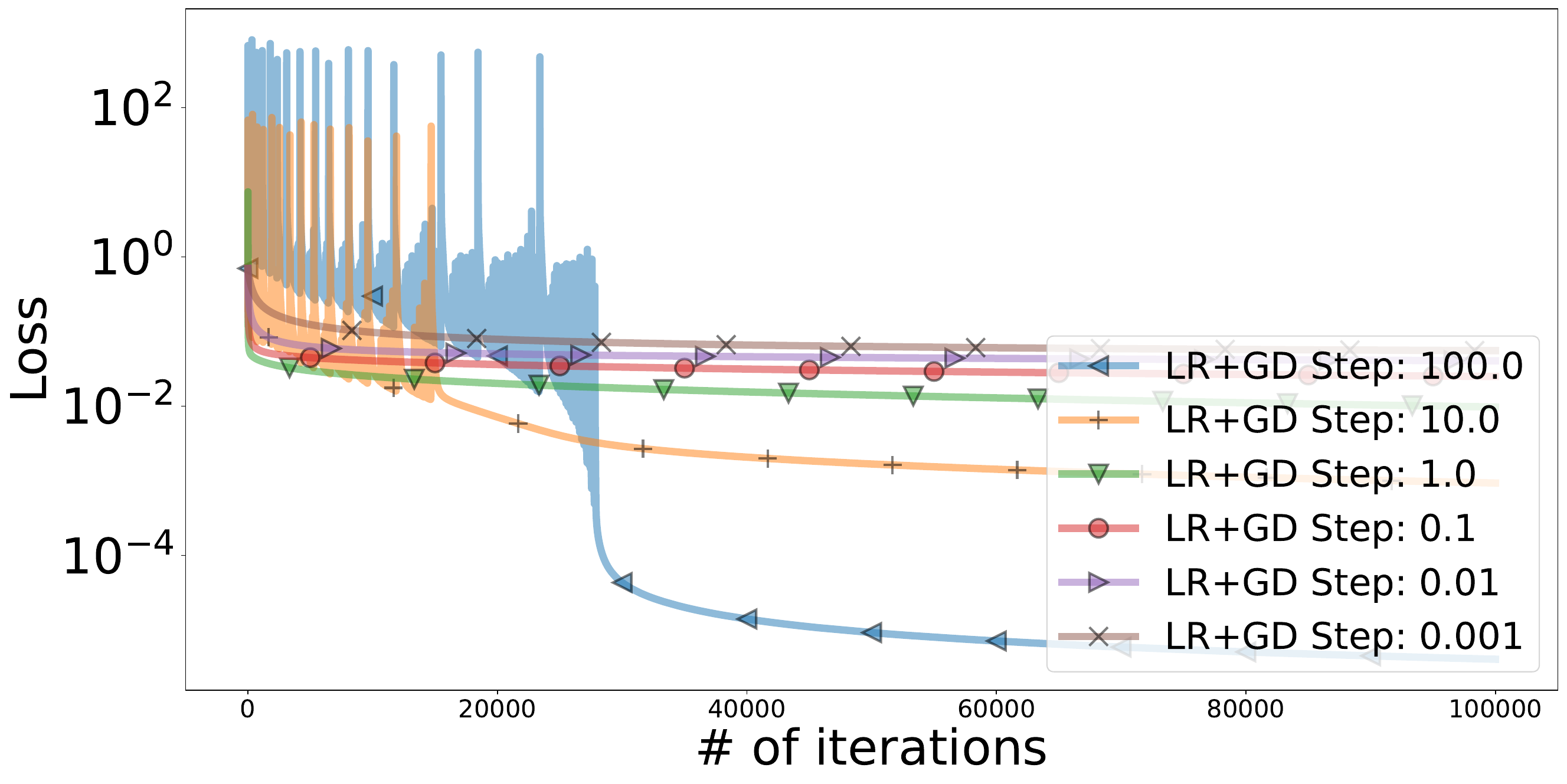}
    \end{subfigure}
    \begin{subfigure}[t]{0.32\textwidth}
        \centering
        \includegraphics[width=\textwidth]{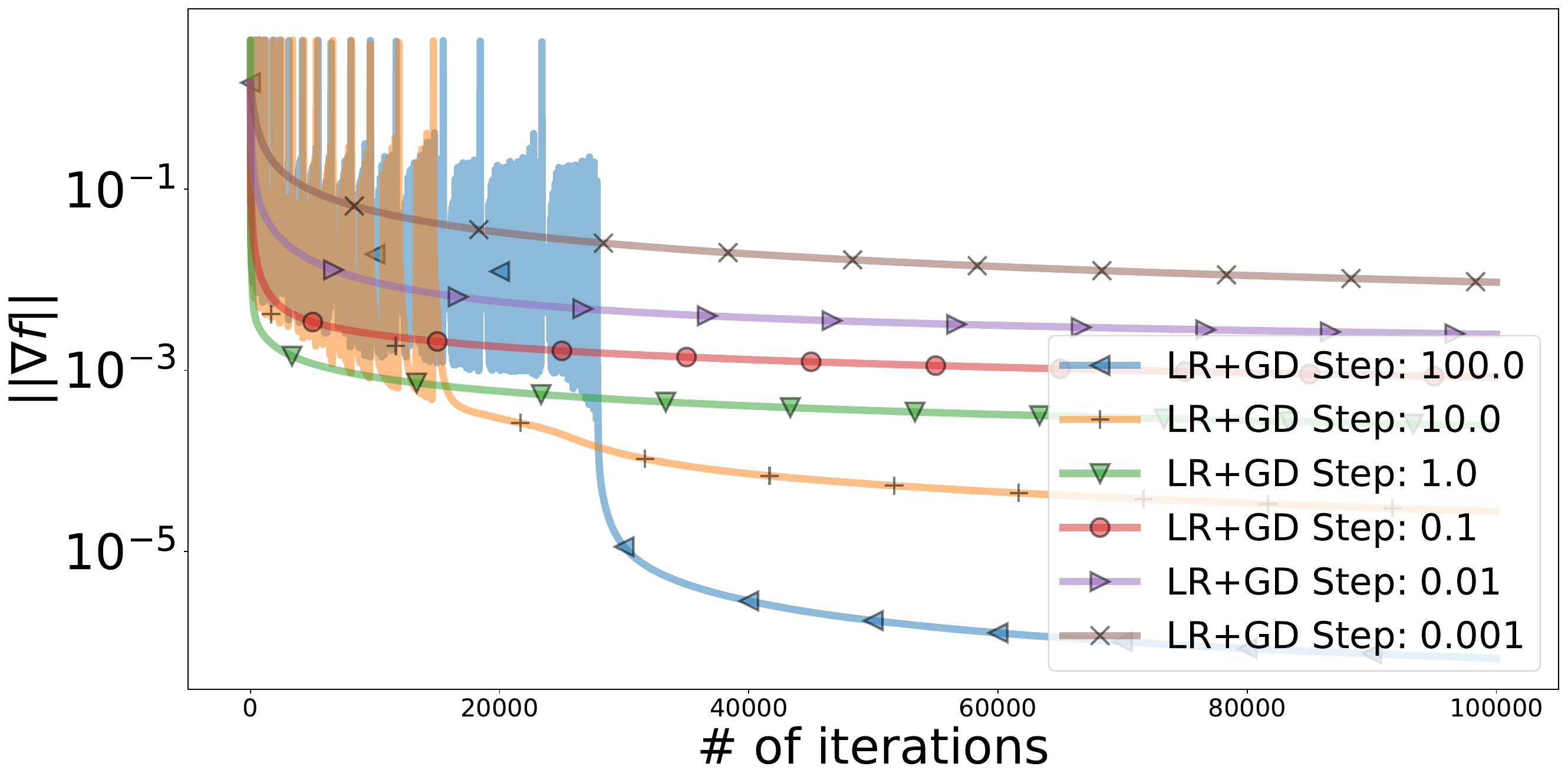}
    \end{subfigure}
    \caption{Accuracy, function values, and the norm of gradients of the logistic loss \eqref{eq:logistic_regression} on \emph{FashionMNIST} but with $10\,000$ samples.}
    \end{figure}

\begin{figure}[H]
\centering
\begin{subfigure}[t]{0.49\textwidth}
    \centering
    \includegraphics[width=\textwidth]{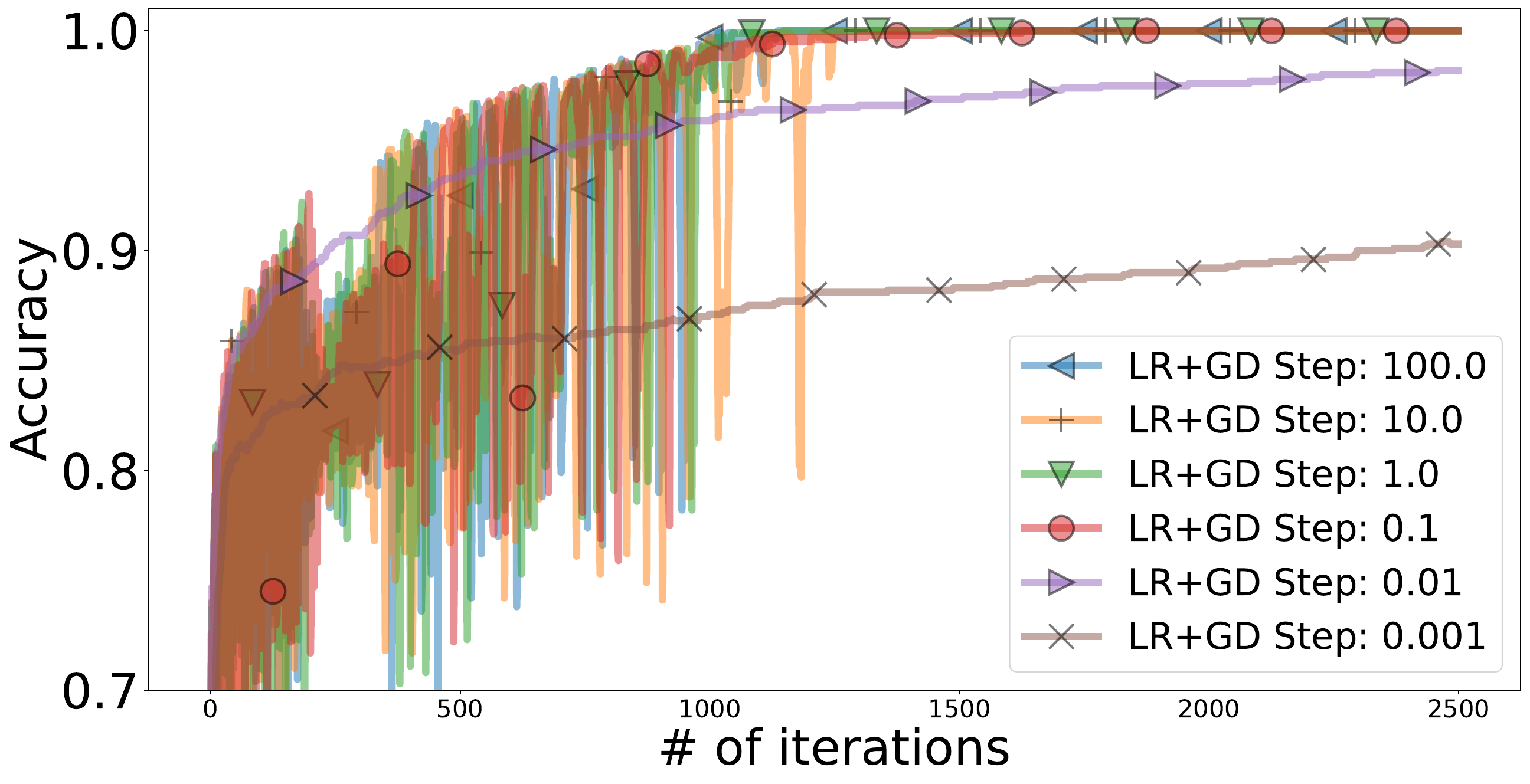}
\end{subfigure}
\begin{subfigure}[t]{0.49\textwidth}
    \centering
    \includegraphics[width=\textwidth]{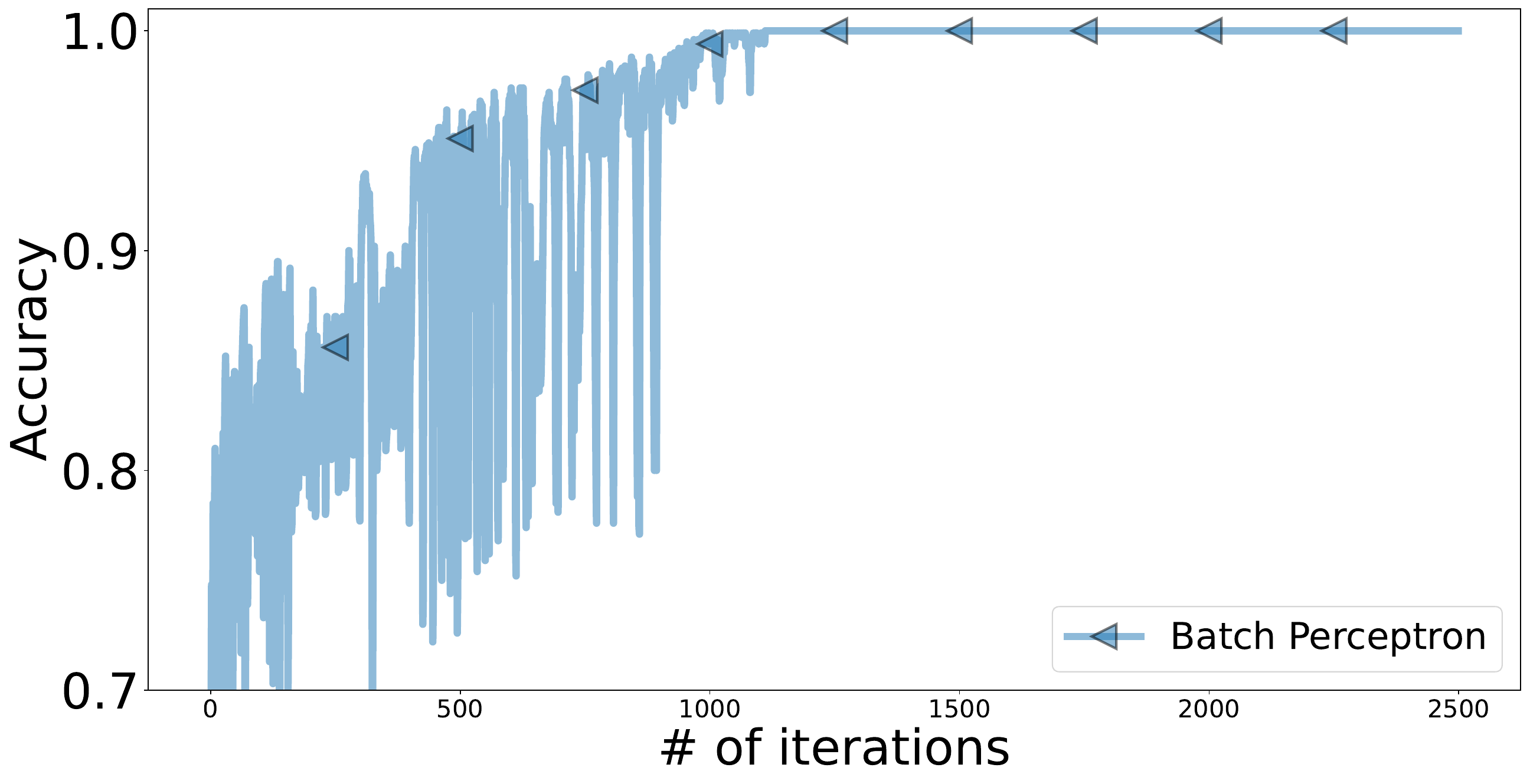}
\end{subfigure}
\caption{In these plots we repeat the experiments from Figure~\ref{fig:cifar10} on \emph{CIFAR-10} but with $1\,000$ samples.}
\centering
\begin{subfigure}[t]{0.32\textwidth}
    \centering
    \includegraphics[width=\textwidth]{./results_2024/gd_two_layer_eos_linear_one_no_bias_loss_bce_logits_num_samples_1000_more_iters_more_iters_filter_classes_0_1_repeat_accuracy.pdf}
\end{subfigure}
\begin{subfigure}[t]{0.32\textwidth}
    \centering
    \includegraphics[width=\textwidth]{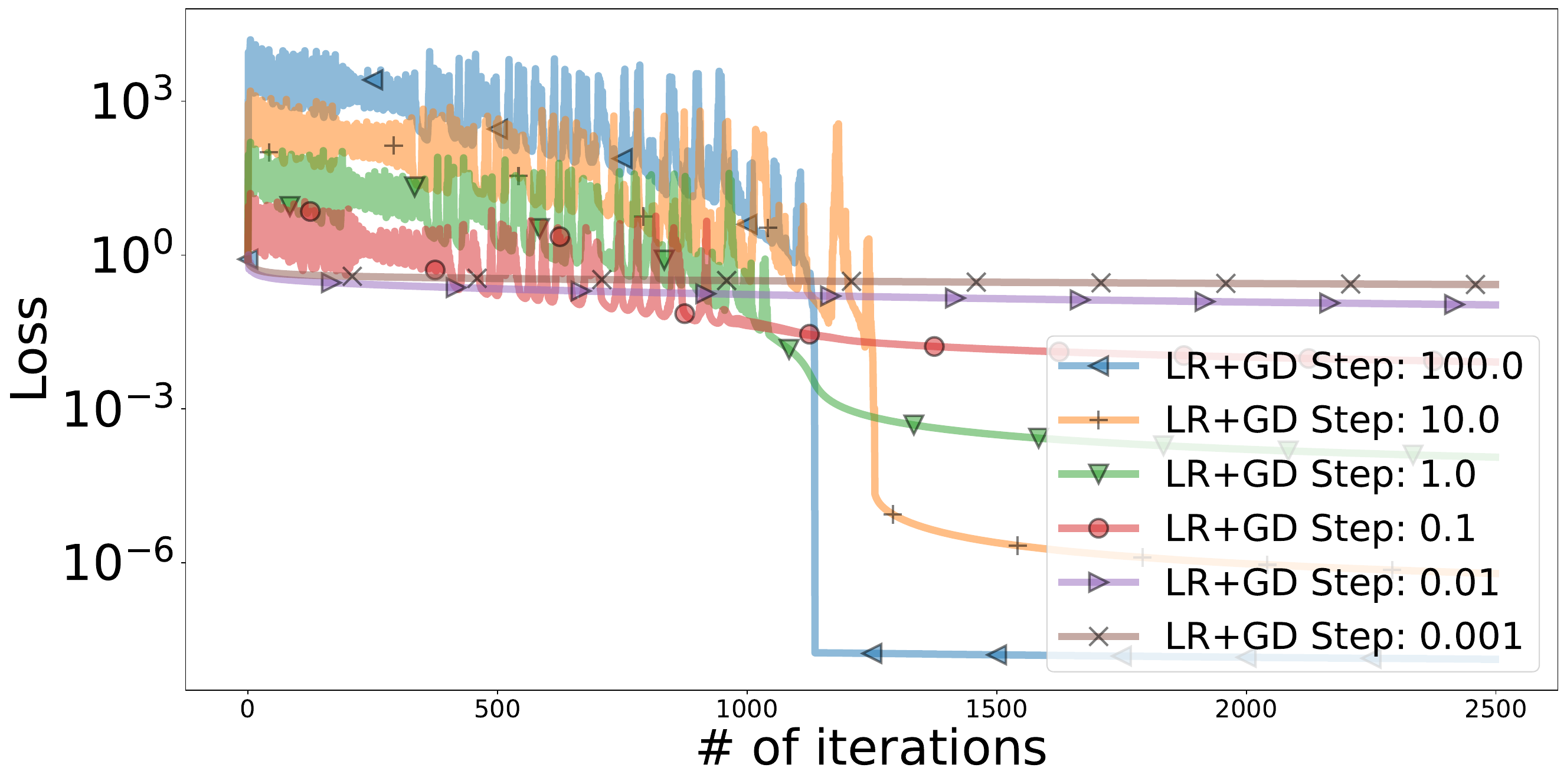}
\end{subfigure}
\begin{subfigure}[t]{0.32\textwidth}
    \centering
    \includegraphics[width=\textwidth]{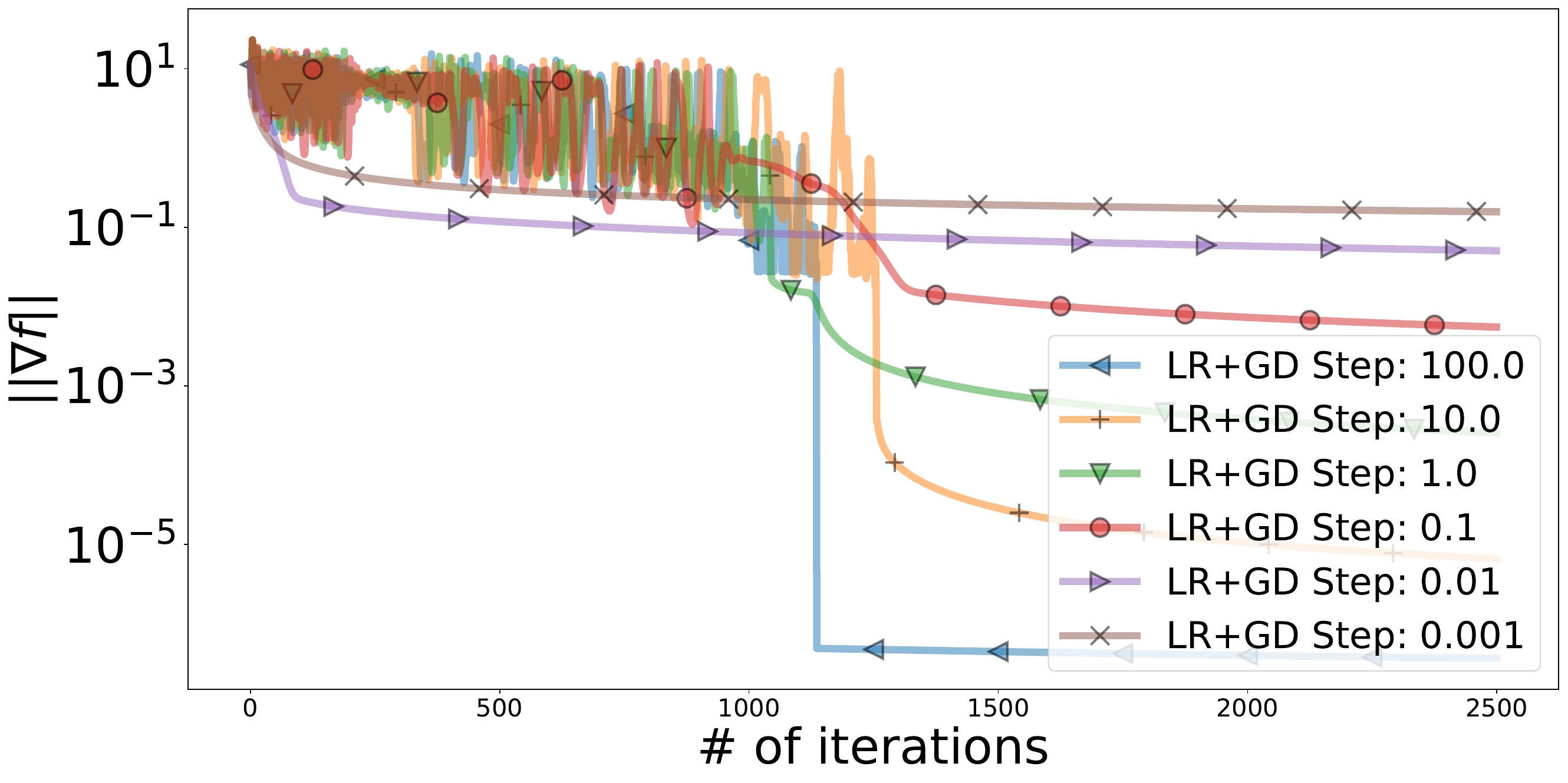}
\end{subfigure}
\caption{Accuracy, function values, and the norm of gradients of the logistic loss \eqref{eq:logistic_regression} on \emph{CIFAR-10} but with $1\,000$ samples.}
\end{figure}
\begin{figure}[h]
\centering
\begin{subfigure}[t]{0.49\textwidth}
    \centering
    \includegraphics[width=\textwidth]{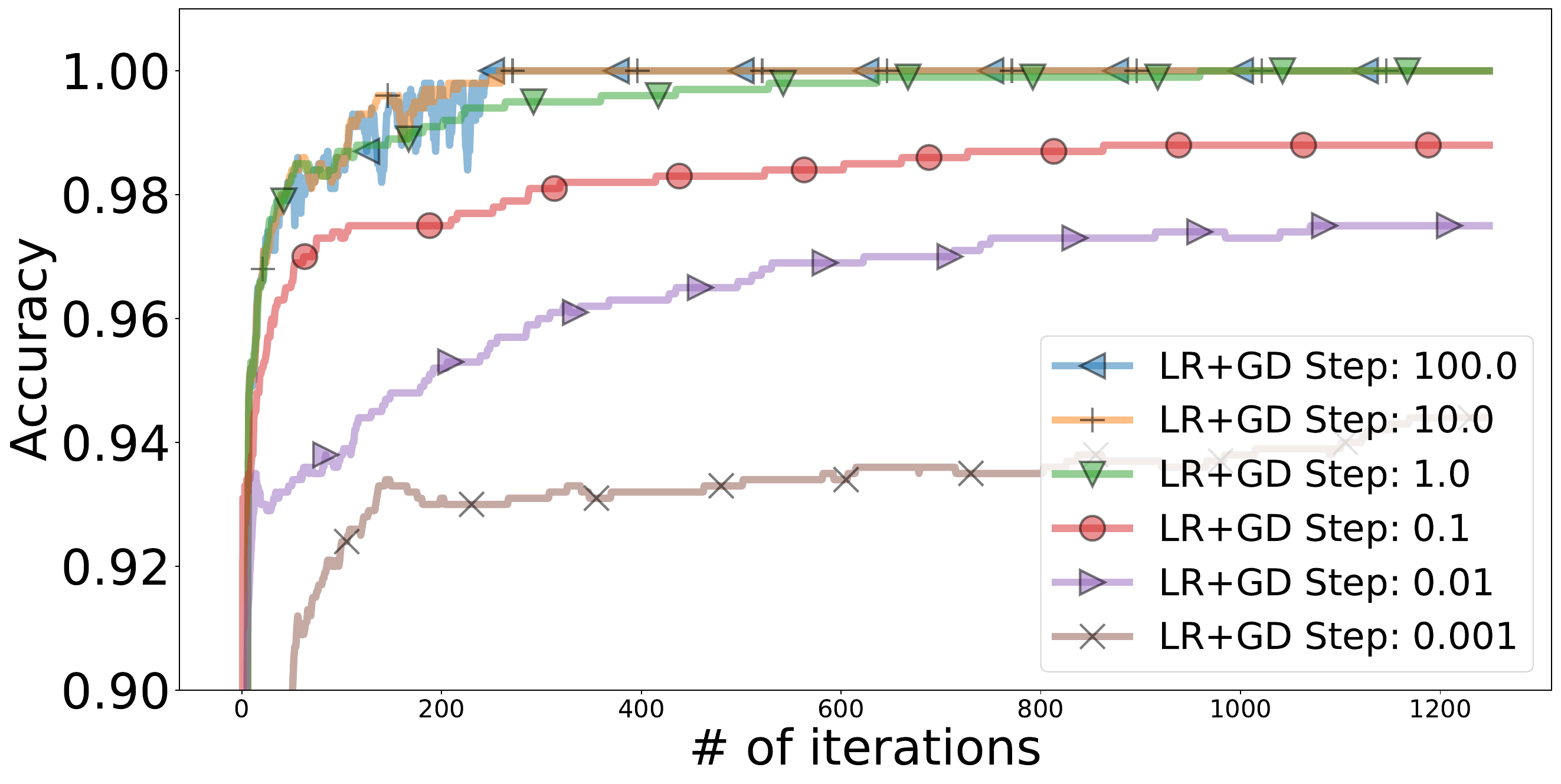}
\end{subfigure}
\begin{subfigure}[t]{0.49\textwidth}
    \centering
    \includegraphics[width=\textwidth]{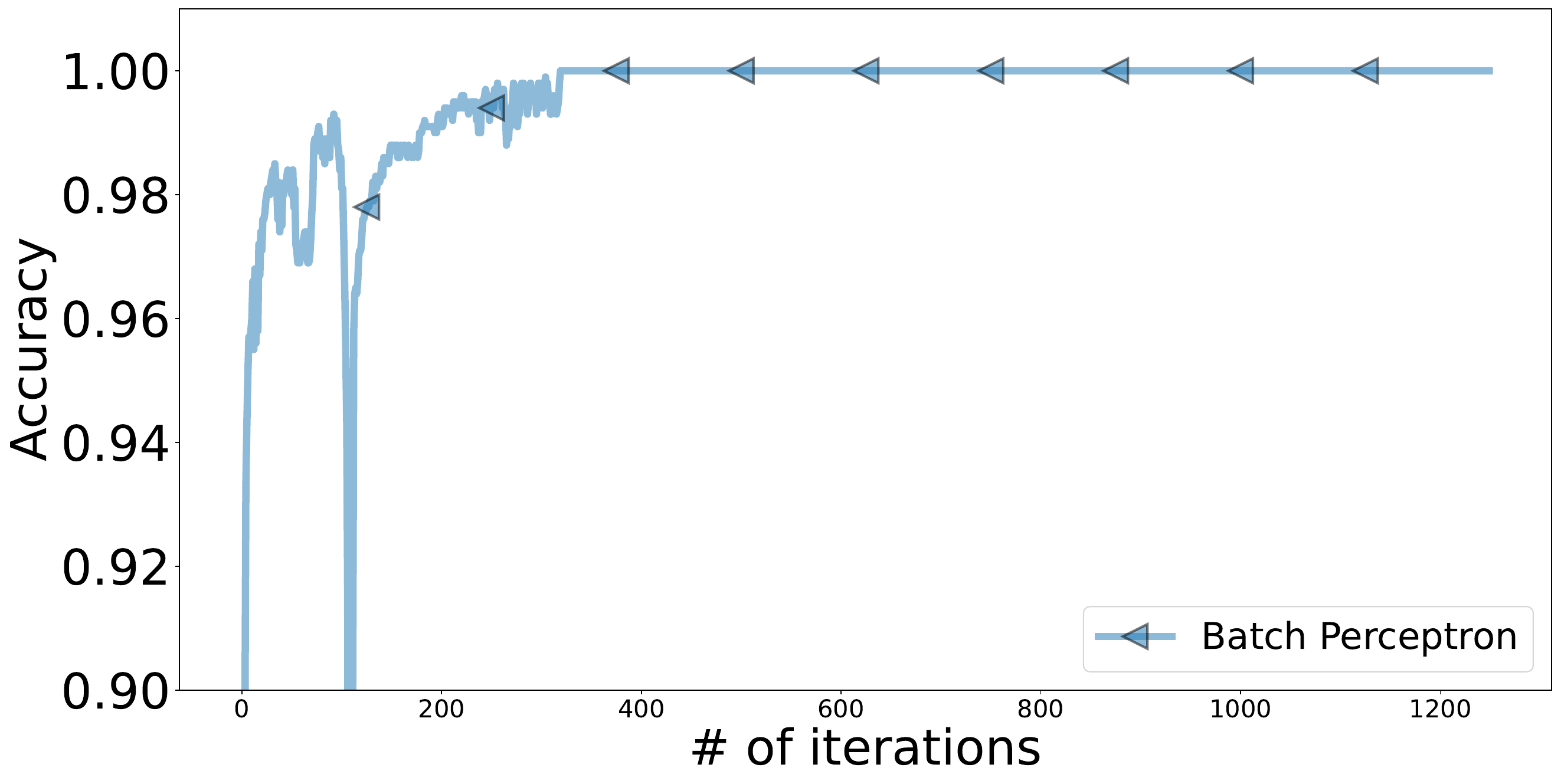}
\end{subfigure}
\caption{In these plots we repeat the experiments from Figure~\ref{fig:fashionmnist} on \emph{FashionMNIST} but with $1\,000$ samples.}
\centering
\begin{subfigure}[t]{0.32\textwidth}
    \centering
    \includegraphics[width=\textwidth]{./results_2024/gd_two_layer_eos_linear_one_no_bias_loss_bce_logits_fashion_mnist_num_samples_1000_more_iters_more_iters_filter_classes_0_4_accuracy.pdf}
\end{subfigure}
\begin{subfigure}[t]{0.32\textwidth}
    \centering
    \includegraphics[width=\textwidth]{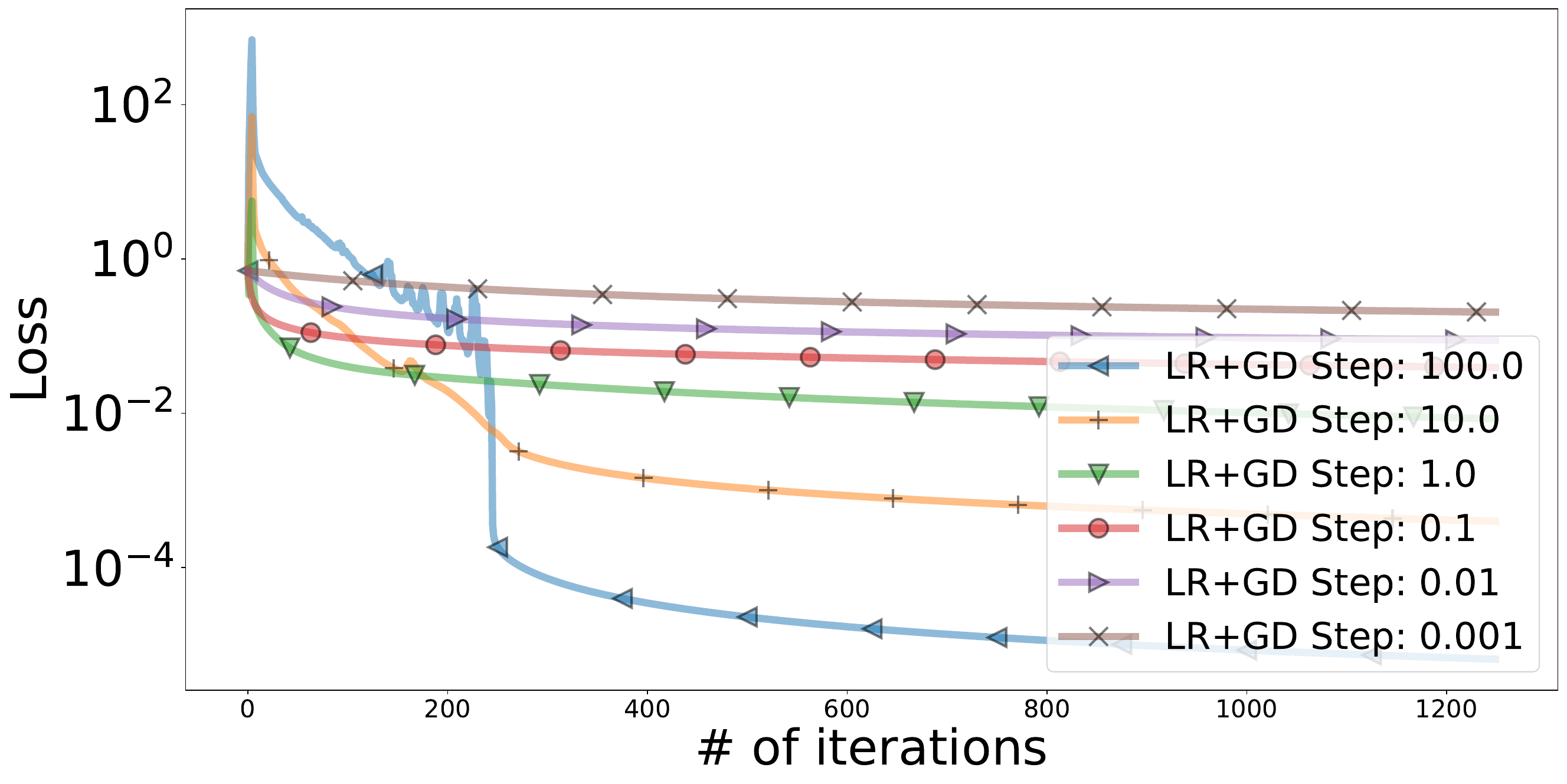}
\end{subfigure}
\begin{subfigure}[t]{0.32\textwidth}
    \centering
    \includegraphics[width=\textwidth]{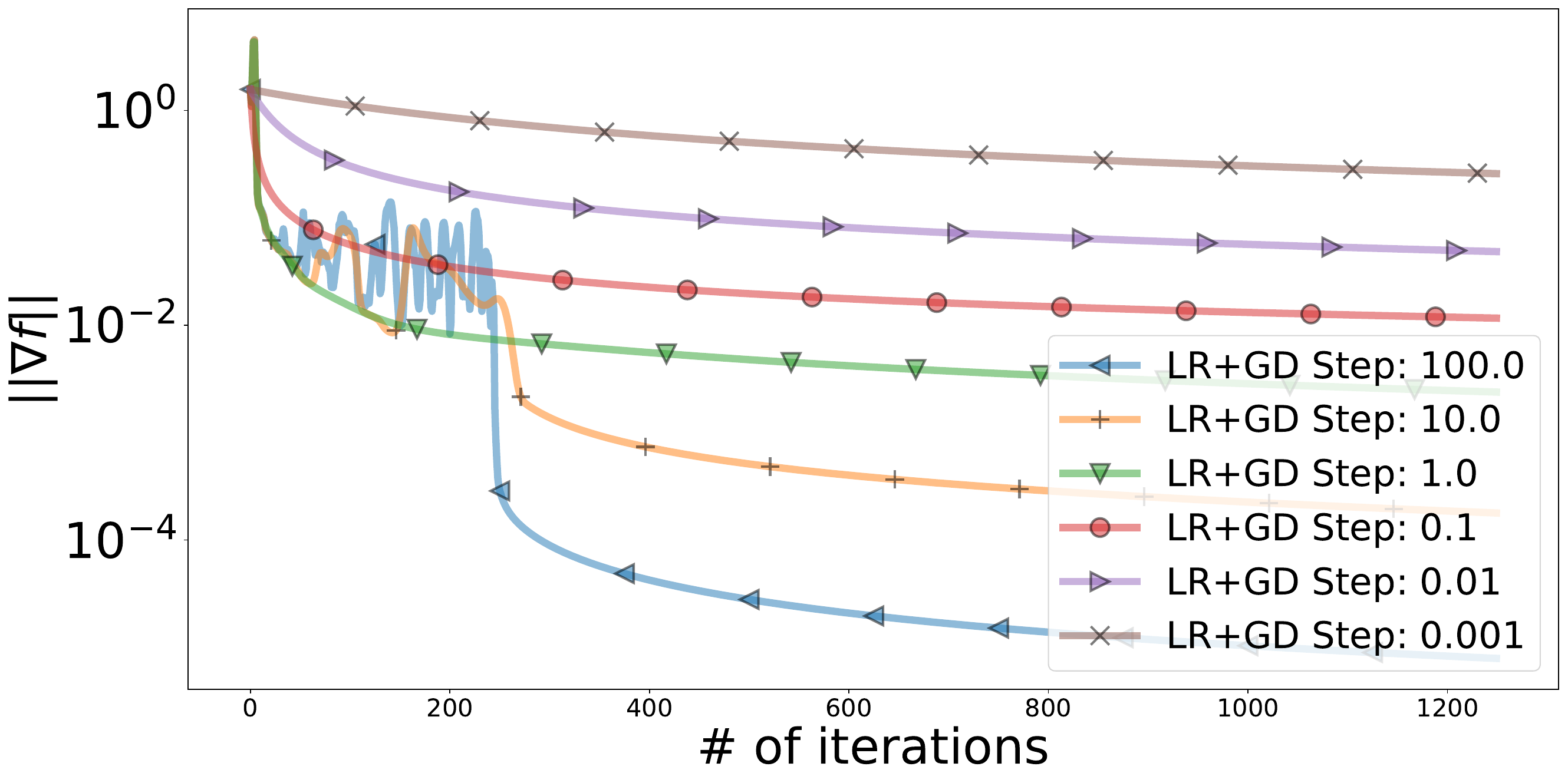}
\end{subfigure}
\caption{Accuracy, function values, and the norm of gradients of the logistic loss \eqref{eq:logistic_regression} on \emph{FashionMNIST} but with $1\,000$ samples.}
\end{figure}

\end{document}

%% file: results_2024/plot_local_optimization_2024_paper_new_algorithm_2_table.tex
\begin{tabular}{|c|c|}
\hline
{Method} & {\# of Iterations} \\
Normalized LR+GD Step: 100.0 & \textbf{2} \\
Normalized LR+GD Step: 10.0 & \textbf{2} \\
Normalized LR+GD Step: 1.0 & 16 \\
Normalized LR+GD Step: 0.1 & 157 \\
LR+GD Step: 100.0 & 308 \\
\hline
\end{tabular}

%% file: results_2024/plot_local_optimization_2024_paper_new_algorithm_1_table.tex
\begin{tabular}{|c|c|}
\hline
{Method} & {\# of Iterations} \\
Normalized LR+GD Step: 100.0 & \textbf{311} \\
Normalized LR+GD Step: 10.0 & 349 \\
Normalized LR+GD Step: 1.0 & 353 \\
Normalized LR+GD Step: 0.1 & 504 \\
LR+GD Step: 100.0 & 538 \\
\hline
\end{tabular}

%% file: results_2024/plot_local_optimization_2024_paper_new_algorithm_3_table.tex
\begin{tabular}{|c|c|}
\hline
{Method} & {\# of Iterations} \\
LR+GD Step: 10.0 & \textbf{2778} \\
LR+GD Step: 100.0 & 2928 \\
Normalized LR+GD Step: 100.0 & 3160 \\
Normalized LR+GD Step: 1.0 & 3351 \\
Normalized LR+GD Step: 10.0 & 3359 \\
\hline
\end{tabular}